\documentclass[letterpaper]{article} 
\usepackage{aaai23}  
\nocopyright
\usepackage{times}  
\usepackage{helvet}  
\usepackage{courier}  
\usepackage[hyphens]{url}  
\usepackage{graphicx} 
\usepackage{natbib}  
\usepackage{caption} 
\title{Calculus on MDPs: Potential Shaping as a Gradient}
\author {
    Erik Jenner,\textsuperscript{\rm 1, 2}
    Herke van Hoof,\textsuperscript{\rm 1}
    Adam Gleave\textsuperscript{\rm 2}
}
\affiliations {
    \textsuperscript{\rm 1} University of Amsterdam\\
    \textsuperscript{\rm 2} UC Berkeley\\
    jenner@berkeley.edu, h.c.vanhoof@uva.nl, gleave@berkeley.edu
}

\usepackage{my_figures}
\usepackage{caption}
\usepackage{subcaption}
\usepackage{booktabs}
\usepackage{csquotes}
\usepackage[shortlabels]{enumitem}

\usepackage{xcolor}
\usepackage[normalem]{ulem}

\usepackage{hyperref}
\hypersetup{
    colorlinks,
    linkcolor={red!50!black},
    citecolor={blue!50!black},
    urlcolor={blue!80!black}
}
\usepackage{my_macros}

\begin{document}
\maketitle

\begin{abstract}
    In reinforcement learning, different reward functions can be equivalent in terms
    of the optimal policies they induce. A particularly well-known and important
    example is potential shaping, a class of functions that can be added to any reward function
    without changing the optimal policy set under arbitrary transition dynamics.
    Potential shaping is conceptually similar to 
    potentials, conservative vector fields and gauge transformations 
    in math and physics, but this connection has not previously been formally explored. We develop a formalism for discrete calculus on
    graphs that abstract a Markov Decision Process, and show how potential shaping
    can be formally interpreted as a gradient within this framework. This allows
    us to strengthen results from \citet{ng1999} describing conditions under which
    potential shaping is the only additive reward transformation to always preserve
    optimal policies. As an additional application of our formalism, 
    we define a rule for picking a single unique
    reward function from each potential shaping equivalence class.
\end{abstract}
\section{Introduction}
For simple tasks, it may be possible to hand-design a good reward function as part of the problem specification.
But in practice, specifying a good reward function is often a significant part of the
challenge~\citep{ng:2000,akrour:2011,christiano:2017,brown:2019,ziegler:2019}.
Reward functions have thus become an important object of study
in their own right~\citep{russell19explaining,michaud2020understanding,gleave2021,jenner2021preprocessing,
icarte2022,skalse2022,wulfe2022}.

One important feature of reward functions is that different reward functions can be \emph{equivalent}
in a strong sense, by inducing identical optimal policies under any transition dynamics.
A trivial example is that scaling any reward function by a positive constant leaves the optimal
policy set unchanged. More interestingly, \citet{ng1999} introduced \emph{potential shaping},
a type of additive reward shaping that is guaranteed to not affect optimal policies.
The name \enquote{\emph{potential} shaping} suggests a connection to potentials and gradient
fields in physics/vector calculus, but this connection
has not previously been formalized and studied.

The central idea of this paper is to construct a version of calculus in which potential
shaping is the gradient of a potential function. Specifically, we generalize
the theory of \emph{discrete calculus on graphs} by introducing a discount factor $\gamma$.
Many concepts from vector calculus carry over to this new framework, allowing us to
describe existing results in the language of calculus, and to derive several new results
related to potential shaping.
Readers do not need prior knowledge about vector calculus or physics to follow this paper.

We have organized our work into two main sections.
First, \cref{sec:calculus} defines the gradient
and curl operators as well as the discounted line integral and gives theoretical results related
to these definitions.
As the main application of these ideas, we generalize a result from the
original potential shaping paper: \citet{ng1999} showed that potential shaping is the only additive shaping
that leaves the optimal policy invariant under arbitrary transition dynamics. We prove that potential
shaping is \emph{still} the only additive shaping term if policy-invariance is required for only a much
smaller set of transition dynamics. Additionally, we use our calculus framework to rederive extensions
of potential shaping to time-dependent~\citep{devlin2012} and action-dependent~\citep{wiewiora2003} potentials.
These follow naturally from an analogy between reinforcement learning and classical gauge theory.
The second main part of the paper, \cref{sec:gauge_fixing} introduces the
divergence operator and shows how rewards with zero divergence can be used
as canonical representatives of each potential shaping equivalence class.
Such canonicalization can for example be used to define
potential shaping-invariant distance measures between reward functions~\citep{gleave2021}.

\section{Related work}
\paragraph{Potential shaping}
By far the most important prior work for this paper is the description of
potential shaping (often also called potential-based reward shaping) by \citet{ng1999}.
They showed that it is the only additive reward transformation that preserves the set
of optimal policies under arbitrary transition dynamics. We extend this result by showing that
potential shaping is still the only additive reward transformation
even if we gain certain types of knowledge about the transition dynamics.

Potential shaping has been generalized to potentials that depend on state-action
pairs~\citep{wiewiora2003} and the current time step~\citep{devlin2012},
as well as both of these at once~\citep{harutyunyan2015}.
Our work recovers these results as natural consequences
of an analogy between potential shaping and gauge transformations.
There are also successful applications of potential shaping to other settings,
such as multi-agent systems~\citep{babes2008,devlin2011}, hierarchical RL~\citep{gao2015},
and multi-objective RL~\citep{mannion2017}.

Recent work has characterized a wider range of reward transformations.
These also do not change optimal policies~\citep{cao2021,kim2021,skalse2022} and various other
objects~\citep{skalse2022,schubert2021} for a subset of possible transition dynamics.
However, we focus exclusively on potential shaping in this work due to its broad applicability.

\paragraph{Discrete calculus on graphs}
Our interpretation of potential shaping through a calculus lens generalizes the theory
of \emph{discrete calculus on graphs}~\citep{grady2010,lim2020}.
Specifically, we interpret the set of possible transitions in an MDP as a directed graph
and then define a generalization of calculus for this graph by taking into account the discount
factor of the MDP.
We do not directly make use of results proven for the undiscounted case,
but most of our definitions and theorem statements are directly inspired by analogous versions
for the undiscounted case.

Graph Laplacians have been applied in reinforcement learning before, e.g.\ using their
eigenvectors as a basis for value function approximation~\citep{mahadevan2007,petrik2007}.
Like our work, this means assigning an underlying graph structure to an MDP, but otherwise
the approaches are quite different. For example, previous work used \emph{undiscounted} Laplacians,
and did not connect these to potential shaping or to discrete calculus more broadly.

\section{Preliminaries}\label{sec:preliminaries}
\subsection{Markov Decision Processes}
A \emph{Markov Decision Process (MDP)} is defined as a tuple $\MDP$ where
$\states$ is a set of states, $\actions$ a set of actions, $P$ a transition distribution
with probability $P(s'|s,a)$ of transitioning to state $s'$ from state $s$ when action $a$ is taken,
$\mu$ is an initial state distribution, $\gamma \in [0,1]$ is a discount factor, and
$R: \states \times \actions \times \states \to \R$ is a reward function that depends on the current
state, action, and next state.

A \emph{trajectory} is a sequence $\tau = (s_0, a_0, s_1, a_1, \ldots)$ of states $s_t \in \states$
and actions $a_t \in \actions$. The \emph{returns} of a trajectory are
$G(\tau) := \sum_{t = 0}^{\abs{\tau} - 1} \gamma^{t} R(s_t, a_t, s_{t+1})$, where $\abs{\tau}$
is the \emph{length} of $\tau$, defined as the number of actions in $\tau$. $\abs{\tau}$ may be finite or infinite.

A \emph{policy} on an MDP is a function $\pi: \states \to \dist(\actions)$, where $\dist(\actions)$
denotes the space of probability distributions over $\actions$. Together with the initial state
distribution $\mu$ and the transition dynamics $P$, a policy induces a distribution over trajectories.
The \emph{expected returns} of a policy $\pi$ are then $G(\pi) := \E_{\tau \sim \mu, \pi, P}[G(\tau)]$.
Throughout \cref{sec:calculus}, we will assume a discounted infinite horizon setting ($\gamma < 1$, $\abs{\tau} = \infty$)
and that the reward function $R$ is bounded. This ensures that the expected returns exist and are
finite.
A policy is called \emph{optimal} if it maximizes expected returns.

\subsection{Potential shaping}
The purpose of the reward function $R$ of an MDP is to specify desirable behavior for the agent,
i.e.\ optimal policies. From that perspective, two reward functions $R$ and $R'$ can be
\emph{equivalent} in the sense that both $R$ and $R'$ induce the same set of optimal policies.
For example, scaling a reward function by a positive constant never affects
which policies are optimal, no matter what the transition dynamics $P$ are. Interestingly,
the same holds for a more complex reward transformation, called \emph{potential shaping}~\citep{ng1999}.

Let $M = \MDP$ be an MDP. A \emph{potential}
on $M$ is a function $\potential: \states \to \R$.
We can use $\potential$ to shape the reward function $R$ as follows:
\begin{equation}
    R'(s, a, s') := R(s, a, s') + \gamma\potential(s') - \potential(s)\,.
\end{equation}
The shaping term $\gamma\potential(s') - \potential(s)$ is called the \emph{potential shaping}
induced by $\potential$.

The returns of a trajectory $\tau = (s_0, a_0, s_1, a_1, \ldots, s_T)$
under the shaped reward function $R'$ are
$G'(\tau) = G(\tau) + \gamma^T \potential(s_T) - \potential(s_0)$,
where $G$ are the returns under $R$. The initial state $s_0$ does not depend
on the policy, and the $\gamma^T \potential(s_T)$ term vanishes if $T \to \infty$,
$\gamma < 1$, and $\potential$ is bounded. So the returns $G'$ and $G$ differ only by a constant that does not
depend on the policy. Thus, $R$ and $R'$ induce identical optimal policies,
no matter what the transition dynamics or initial state distribution of the MDP are.

\section{Calculus with discounting}\label{sec:calculus}

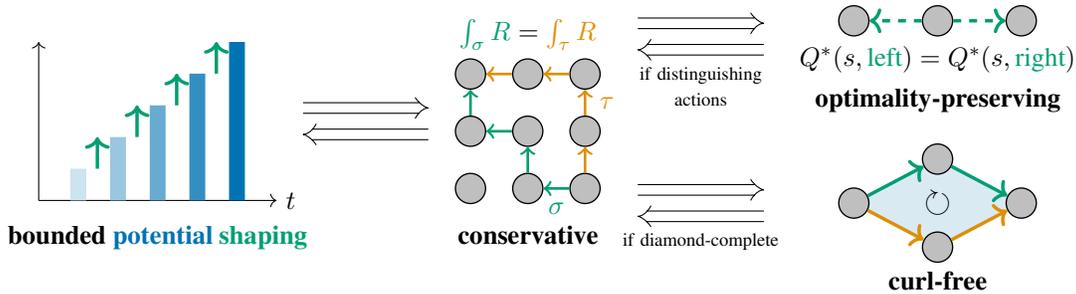
\begin{figure*}
    \centering
      \begin{tikzpicture}
        \newcommand{\implarrow}[1]{
            \draw ([shift={(0pt,5pt)}]A) coordinate(a) to[Double] (a-|B);
            \draw ([shift={(0pt,-5pt)}]A-|B) coordinate(a) to[Double] (a-|A);
            \node[impltext] at (M) {#1};
        }
            \csdef{node11v}{100}
            \csdef{node12v}{60}
            \csdef{node13v}{80}
            \csdef{node21v}{80}
            \csdef{node22v}{60}
            \csdef{node23v}{30}
            \csdef{node31v}{20}
            \csdef{node32v}{40}
            \csdef{node33v}{20}
        \begin{scope}[x=15pt,y=0.6pt,shift={(0pt,0pt)}]
            \draw[->] (0,0) coordinate(O) --++(0:6) coordinate(E) node[at end,right]{$t$};
            \coordinate (M) at ($(O)!0.5!(E)$);
            \coordinate (R1) at (6,50);
            \draw[->] (0,0)--++(90:100);
            \setcounter{barpos}{0}
            \foreach \n in {33,32,22,21,11} {\stepcounter{barpos}
                \draw[draw=COLint!\csuse{node\n v},bar] (\thebarpos,0)--++(90:\csuse{node\n v}) coordinate (max\thebarpos);
                \coordinate (pos\thebarpos) at (\thebarpos.48,0);
            }
            \setcounter{barpos}{1}
            \foreach \n in {1,2,3,4} {\stepcounter{barpos}\draw[path2] (pos\n|-max\n)--(pos\n|-max\thebarpos);}
            \node[figlabel] at ([shift={(0,-14pt)}]M) {bounded \textcolor{COLint}{potential} \textcolor{COLpath2}{shaping}};
        \end{scope}
        \begin{scope}[shift={(185pt,26.2pt)}]
                \matrix (m) [matrix of nodes,nodes in empty cells,row sep=15pt,column sep=15pt] {& & \\ & & \\ & & \\};
                \foreach \row in {1,2,3} {\foreach \col in {1,2,3} {\node[graphnodee] (n-\row-\col) at (m-\row-\col) {\strut};}}
                \coordinate (M) at ($(m-3-1.south)!0.5!(m-3-3.south)$);
                \coordinate (L2) at (n-2-1.west);
                \coordinate (R2) at (n-2-3.east);
                \node at ([yshift=15pt]n-1-2) {$\color{COLpath2}\int_\sigma R \color{black} = \color{COLpath1}\int_\tau R$};
                \foreach \n[remember=\n as \lastn (initially 3-3)] in {2-3,1-3,1-2,1-1} {\draw[path1,overpath] (n-\lastn)--(n-\n);}
                \foreach \n[remember=\n as \lastn (initially 3-3)] in {3-2,2-2,2-1,1-1} {\draw[path2,overpath] (n-\lastn)--(n-\n);}
                \path (n-2-3)--(n-1-3) node[midway,right,xshift=2pt]{\color{COLpath1}$\tau$};
                \path (n-3-3)--(n-3-2) node[midway,below,yshift=-2pt]{\color{COLpath2}$\sigma$};
                \node[figlabel] at ([shift={(0,-14pt)}]M) {conservative};
        \end{scope}
        \begin{scope}[shift={(340pt,68pt)}]
                \matrix (m) [matrix of nodes,nodes in empty cells,row sep=10pt,column sep=25pt] {& & \\};
                \foreach \n in {1,2,3} {\node[graphnodee] (n-1-\n) at (m-1-\n) {\strut};}
                \node[below] (Q) at (n-1-2.south) {$Q^*(s,\text{\textcolor{COLpath2}{left}})=Q^*(s,\text{\textcolor{COLpath2}{right}})$};
                \coordinate (M) at (Q.south);
                \coordinate (L3) at (Q.north west);
                \foreach \n in {1,3} {\draw[path2,dashed] (n-1-2)--(n-1-\n);}
                \node[figlabel] at ([shift={(0,-7pt)}]M) {optimality-preserving};
        \end{scope}
        \begin{scope}[shift={(340pt,-1pt)}]
                \matrix (m) [matrix of nodes,nodes in empty cells,row sep=10pt,column sep=25pt] {& & \\ & & \\ & & \\};
                \foreach \n in {1-2,3-2,2-1,2-3} {\node[graphnodee] (n-\n) at (m-\n) {\strut};}
                \coordinate (M) at (m-3-2.south);
                \coordinate (L4) at (n-2-1-|L3);
                \node[circle,scale=1.3,inner sep=0.pt] (n-2-2) at (m-2-2) {$\circlearrowright$};
                \foreach \n[remember=\n as \lastn (initially 2-1)] in {3-2,2-3} {\draw[path1] (n-\lastn)--(n-\n);}
                \foreach \n[remember=\n as \lastn (initially 2-1)] in {1-2,2-3} {\draw[path2] (n-\lastn)--(n-\n);}
                \begin{scope}[on background layer]
                    \fill[evarea] (n-2-1.center)--(n-3-2.center)--(n-2-3.center)--(n-1-2.center)--cycle;
                \end{scope}
                \node[figlabel] at ([shift={(0,-9pt)}]M) {curl-free};
        \end{scope}
        \path ([shift={(10pt,0pt)}]R1) -- ([shift={(-10pt,0pt)}]R1-|L2) coordinate[at start] (A) coordinate[at end] (B) coordinate[midway] (M);
        \implarrow{};
        \path ([shift={(14pt,0pt)}]R2|-L3) coordinate (r1) -- ([shift={(-10pt,0pt)}]r1-|L3) coordinate (r2) coordinate[at start] (A) coordinate[at end] (B) coordinate[midway] (M);
        \implarrow{if distinguishing\\ actions};
        \path (r1|-L4) -- (r2|-L4) coordinate[at start] (A) coordinate[at end] (B) coordinate[midway] (M);
        \implarrow{if diamond-complete};
    \end{tikzpicture}%

    \caption{Reward functions can have several closely related properties: they may be a pure potential
      shaping term, be conservative, optimality-preserving, and curl-free. The implications from left to right always hold. The implications from right to left hold under various assumptions on the set of allowed
      transitions.}\label{fig:overview}
\end{figure*}

The name \enquote{\emph{potential} shaping} is inspired by physics.
There, the gradient of a potential is used to describe \emph{conservative force fields},
where the potential energy needed to traverse a trajectory does not depend on the precise
trajectory, only on the start and end point. Analogously, only the first and last state
of a trajectory in an MDP affect the returns of the potential shaping term.

The goal of this section is to formalize this analogy by introducing a framework
of calculus in which potential shaping is the gradient of the potential, and thus is
directly analogous to a conservative vector field. To do so,
we generalize the theory of \emph{discrete calculus on graphs} by introducing
a discount factor. More specifically, the gradient operator on graphs is usually
defined as
\begin{equation}
    (\grad_{\text{undiscounted}} \potential)(s, s') = \potential(s') - \potential(s)\,,
\end{equation}
where $\potential$ is a function on the nodes of the graph, and $\grad_{\text{undiscounted}} \potential$
is a function on edges. This is the graph equivalent of the gradient $\nabla \Phi$ of a scalar field $\Phi$
in vector calculus.
It already looks very similar to the potential shaping term
$\gamma\potential(s') - \potential(s)$; the only
difference is the additional discount factor in the first term.

An important part of our contribution is the introduction of the formal framework
for the calculus of reward functions itself. But we also use this framework to derive
results that generalize those from the original potential shaping paper~\citep{ng1999}.
\Cref{fig:overview} summarizes these new results: we consider four different properties
that reward functions can have and show under which conditions various implications
between them hold. Within this framework, we prove the result by \citet{ng1999} that
\enquote{potential shaping $\iff$ optimality-preserving} under weaker assumptions
than the original paper required.

In the following subsections, we will introduce our framework step by step, while defining
the various terms in \cref{fig:overview} and formally stating the implications once the
necessary machinery is in place.

\subsection{Transition graphs}

We can turn an MDP into a graph-like structure by using the states as nodes, and possible transitions
as directed edges. There may be multiple edges between two nodes, indexed by different actions.
For our framework of calculus with discounting, we also need to include the discount factor.
That gives us the structure we will define our framework on:
\begin{definition}
    A \emph{transition graph} is a tuple $(\states, \actions, \transitions, \gamma)$, where
    $\states$ and $\actions$ are arbitrary sets, called the \emph{state space} and \emph{action space} respectively,
    $\transitions \subseteq \states \times \actions \times \states$ is the set of
    \emph{allowed transitions}, and
    $\gamma \in [0, 1]$ is the \emph{discount factor}.
\end{definition}

Because we use the infinite-horizon setting, we will assume throughout this section that $\gamma < 1$
and that every state has at least one outgoing transition (i.e.\ there are no \enquote{dead ends}).
This is no real restriction since terminal states can always be modeled as states with a self-loop
to move from the finite-horizon to the infinite-horizon setting.

We say that an MDP $\MDP$ is \emph{compatible} with a transition
graph $(\states, \actions, \transitions, \gamma)$ if $P(s'|s, a) = 0$ for all $(s, a, s') \notin \transitions$.
Note that we do not require the converse: $P(s'|s, a)$ may be zero for some $(s, a, s') \in \transitions$.
Intuitively, we can think of $\transitions$ as representing partial knowledge about
the transition dynamics: we know that certain transitions are impossible (namely those not in $\transitions$),
but otherwise make no restrictions.
\Cref{fig:compatibility} shows this formalism applied to a simple $2\times 2$ gridworld.

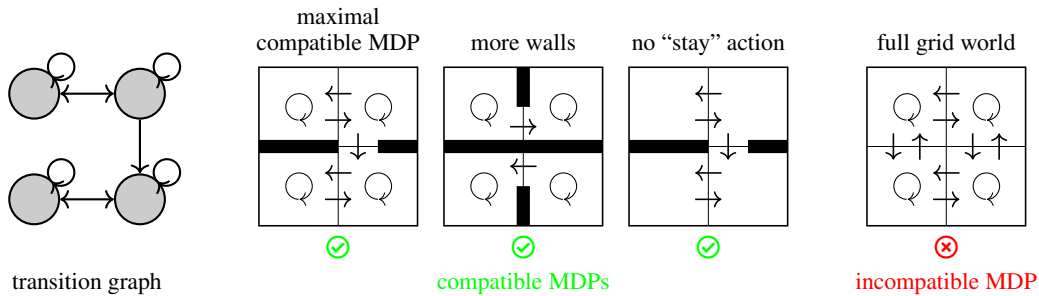
\begin{figure*}[!ht]\centering
      \begin{tikzpicture}
        \newlength{\looprad}\setlength{\looprad}{5pt}
        \newcommand{\selfloop}[1]{
            \coordinate (SL) at (#1);
            \draw[edge,->] ([shift={(45:-0.5\edgelw)}]#1.45) arc [start angle=225,end angle=-135,radius=\looprad];
        }
        \newcommand{\selfloops}{
            \foreach \x in {0,1}
                {\foreach \y in {0,1} {
                    \coordinate (SL) at (\x,\y);
                    \draw[->] ([yshift=-\looprad]SL) arc[start angle=268,end angle=-88,radius=\looprad];
                }}
        }
        \newcommand{\house}[2]{
            \coordinate (Hbl) at (-0.5,-0.5);
            \coordinate (Htr) at (1.5,1.5);
            \coordinate (Hbr) at (Hbl-|Htr);
            \coordinate (Htl) at (Hbl|-Htr);
            \coordinate (C) at ($(Hbl)!0.5!(Htr)$);
            \draw[line width=0.6pt] (Hbl) rectangle (Htr);
            \draw[line width=0.3pt] ($(Hbl)!0.5!(Htl)$) coordinate (L)--($(Hbr)!0.5!(Htr)$) coordinate (R);
            \draw[line width=0.3pt] ($(Hbl)!0.5!(Hbr)$) coordinate (B)--($(Htl)!0.5!(Htr)$) coordinate (T);
            \coordinate (l) at ($(L)!0.5!(C)$);\coordinate (l1) at ($(L)!1/3!(C)$);\coordinate (l2) at ($(L)!2/3!(C)$);\coordinate (l3) at ($(L)!3/4!(C)$);
            \coordinate (r) at ($(R)!0.5!(C)$);\coordinate (r1) at ($(R)!1/3!(C)$);\coordinate (r2) at ($(R)!2/3!(C)$);\coordinate (r3) at ($(R)!3/4!(C)$);
            \coordinate (b) at ($(B)!0.5!(C)$);\coordinate (b1) at ($(B)!1/3!(C)$);\coordinate (b2) at ($(B)!2/3!(C)$);\coordinate (b3) at ($(B)!3/4!(C)$);
            \coordinate (t) at ($(T)!0.5!(C)$);\coordinate (t1) at ($(T)!1/3!(C)$);\coordinate (t2) at ($(T)!2/3!(C)$);\coordinate (t3) at ($(T)!3/4!(C)$);
            \node[above,font=\footnotesize,align=center] at (T) {\strut#1\strut};
            \ifstrequal{#2}{Y}
                {\def\MDPcheck{\color{green}{\faCheckCircle[regular]}}}
                {\def\MDPcheck{\color{red}{\faTimesCircle[regular]}}}
            \node[below,font=\footnotesize,align=center] at (B) {\strut\MDPcheck\strut};
        }
        \newcommand{\hdir}[2]{\draw[line width=0.6pt,#1] ([xshift=-\looprad]#2)--++(0:2\looprad);}
        \newcommand{\vdir}[2]{\draw[line width=0.6pt,#1] ([yshift=-\looprad]#2)--++(90:2\looprad);}
        \begin{scope}[x=40pt,y=40pt,xshift={-21\looprad},yshift=-\looprad]
            \node[emptyvertex,fill=black!20] (V00) at (0,0) {};
            \node[emptyvertex,fill=black!20] (V01) at (0,1) {};
            \node[emptyvertex,fill=black!20] (V10) at (1,0) {};
            \node[emptyvertex,fill=black!20] (V11) at (1,1) {};
            \foreach \P in {00,01,10,11} {\selfloop{V\P}}
            \draw[edge,<->] (V00)--(V10);
            \draw[edge,<->] (V01)--(V11);
            \draw[edge,->] (V11)--(V10);
            \node[font=\footnotesize] (BN) at (0.5,-0.8) {\strut transition graph};
        \end{scope}
        \begin{scope}[x=6\looprad,y=6\looprad,xshift={-1\looprad}]
            \house{maximal\\ compatible MDP}{Y}
            \selfloops
            \draw[room] (L)--(C);
            \draw[room] (r)--(R);
            \hdir{<-}{t1}
            \hdir{->}{t2}
            \vdir{<-}{r3}
            \hdir{->}{b1}
            \hdir{<-}{b2}
        \end{scope}
        \begin{scope}[x=6\looprad,y=6\looprad,xshift={13\looprad}]
            \house{more walls}{Y}
            \selfloops
            \draw[room] (L)--(R);
            \draw[room] (T)--(t);
            \draw[room] (B)--(b);
            \hdir{->}{t3}
            \hdir{<-}{b3}
            \coordinate (BnodeY) at (B);
        \end{scope}
        \begin{scope}[x=6\looprad,y=6\looprad,xshift={27\looprad}]
            \house{no ``stay'' action}{Y}
            \draw[room] (L)--(C);
            \draw[room] (r)--(R);
            \hdir{<-}{t1}
            \hdir{->}{t2}
            \vdir{<-}{r3}
            \hdir{->}{b1}
            \hdir{<-}{b2}
        \end{scope}
        \begin{scope}[x=6\looprad,y=6\looprad,xshift={45\looprad}]
            \house{full grid world}{N}
            \selfloops
            \hdir{<-}{t1}
            \hdir{->}{t2}
            \hdir{->}{b1}
            \hdir{<-}{b2}
            \vdir{->}{r1}
            \vdir{<-}{r2}
            \vdir{<-}{l1}
            \vdir{->}{l2}
            \coordinate (BnodeN) at (B);
        \end{scope}
        \node[font=\footnotesize,green] at (BN-|BnodeY) {\strut compatible MDPs};
        \node[font=\footnotesize,red] at (BN-|BnodeN) {\strut incompatible MDP};
    \end{tikzpicture}%

    \caption{
        A simple $2 \times 2$ gridworld example of transition graphs. The transition graph
        specifies that direct transitions between the two left states are impossible, and that it is
        impossible to go upwards on the right. All compatible MDPs reproduce these restrictions, but
        may also add additional ones. More subtly, the transition graph also forbids any diagonal
        transitions (which would correspond to \enquote{teleportation} in a gridworld). Not shown in the
        figure is that MDPs will also have to specify a reward function and initial state
        distribution---these have nothing to do with compatibility.
    }\label{fig:compatibility}
\end{figure*}

Just like in MDPs, a \emph{potential} on a transition graph is a function
$\potential: \states \to \R$.
By contrast, a \emph{reward function} on a transition graph is a function
$R: \transitions \to \R$.
Note this is defined only on the allowed transitions $\transitions$,
not on all of $\states \times \actions \times \states$: if we know that a transition
is impossible, we do not need to specify its reward.
We write $\scalars$ and $\vectors$ for the space of potentials and reward functions respectively.

\subsection{The discounted gradient and line integral}
As we have seen above, the potential shaping term is very similar to the gradient
in calculus on graphs, differing only in the discount factor. So we will refer
to this discounted version as a \enquote{gradient} as well---we will see repeatedly
that this name is very much justified by similarities to the undiscounted gradient.
\begin{definition}
    The \emph{gradient} on a transition graph is the operator
    $\grad: \scalars \to \vectors$ defined by
    \begin{equation} \label{eq:grad_transition}
        (\grad \potential)(s, a, s') := \gamma \potential(s') - \potential(s)\,.
    \end{equation}
\end{definition}
Note that the gradient does not depend on the action $a$, just like the potential shaping term.
In terms of the graph, the gradient depends only on the origin and target of an edge, not on the edge itself.

Given that the gradient operator on Euclidean space maps from scalar fields to vector fields,
this definition suggests thinking of potentials as scalar fields and
reward functions as vector fields. Potentials as scalar fields is intuitively reasonable.
The \enquote{reward functions = vector fields} analogy is more subtle.
Define the reward function
$R: \states \times \actions \times \states \to \R$ to instead be a function with signature
$R: \states \to \R^{\actions \times \states}$.
In a tabular setting, the function space $\R^{\actions \times \states}$ is isomorphic
to the Euclidean space $\R^{\abs{\actions}\abs{\states}}$, so we can think of a reward function
as assigning a vector in Euclidean space to every state---that is precisely a vector field
over states.
This is only an informal motivation. We refer interested readers to \cref{sec:vectors}
for a more formal explanation of how reward functions are related to vector fields, as well
as some discrepancies.

To define the analogue of a conservative vector field in our formalism, we require the concept of a \emph{line integral}:
\begin{definition}
    Let $\tau = (s_0, a_0, s_1, \ldots)$ be a trajectory in a transition graph with discount factor $\gamma$, and let
    $R$ be a reward function on that transition graph.
    Then the (discounted) \emph{line integral} of $R$ over $\tau$ is defined as
    \begin{equation}
        \int_\tau R := \sum_{t = 0}^{\abs{\tau} - 1} \gamma^t R(s_t, a_t, s_{t + 1})\,.
    \end{equation}
\end{definition}
Throughout \cref{sec:calculus}, we will assume that $\gamma < 1$ and that all reward functions
are bounded. This ensures that even infinite line integrals always exist.
Notice that the line integral simply corresponds to the returns in reinforcement learning: $G(\tau) = \int_{\tau} R$.
One important note is that we only allow trajectories $\tau$ where each transition $(s_t, a_t, s_{t + 1})$
is in the set of allowed transitions $\transitions$. Otherwise, the line integral along
$\tau$ would not be defined, since reward functions are only defined on $\transitions$.

\subsection{Conservative reward functions}\label{sec:conservative}
We have hinted at the idea that potential shaping is related to \emph{conservative vector fields}
in vector calculus and shall now make this connection more precise. A vector field is conservative
if its line integrals depend only on start and end point, not on the path itself.
For our purposes, a slightly different definition turns out to be more suitable: in the infinite-horizon
RL setting we use, we are mainly interested in \emph{infinite} trajectories, which don't have an end point.
We will therefore use the following definition:
\begin{definition}
    A reward function $R$ is called \emph{conservative} if for any two infinite trajectories
    $\sigma, \tau$ with the same initial state, we have $\int_\sigma R = \int_\tau R$.
\end{definition}
Recall that we assume $\gamma < 1$ and bounded reward functions, so the integrals always exist.

More directly analogous to conservative vector fields would be the following notion:
\begin{restatable}{definition}{finitelyConservativeDef}
    A reward function $R$ is called \emph{finitely conservative} if for any two finite trajectories
    $\sigma, \tau$ with the same initial state, final state, and length, we have
    $\int_\sigma R = \int_\tau R$.
\end{restatable}
We need the length of $\sigma$ and $\tau$ to be the same because of
the discount factor that isn't present in vector calculus. Without this assumption,
finite conservativeness would almost never be satisfied and thus be a mostly useless concept.
Finite conservativeness is a weaker notion than the infinite version:
\begin{restatable}{proposition}{finitelyConservative}\label{thm:finitely_conservative}
    Conservative reward functions are finitely conservative.
\end{restatable}

In this paper, we mainly focus on (infinite) conservativeness since this simplifies some of the theory.
However, the \enquote{spirit} of the results is still similar using finite conservativeness
and in \cref{sec:finitely_conservative}, we discuss the relation between the two concepts in more
detail. Because of \cref{thm:finitely_conservative}, we do state results in the main paper using
finite conservativeness if possible whenever that makes them stronger.

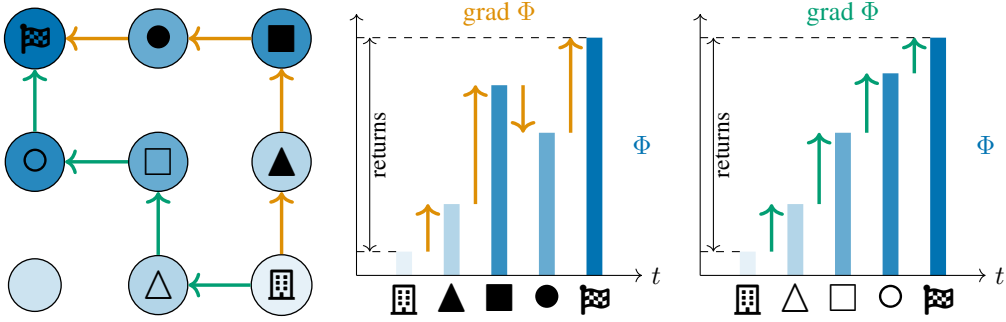
\begin{figure*}
    \centering
      \begin{tikzpicture}
            \csdef{node11i}{\resizebox{!}{1em}{\faFlagCheckered}}     \csdef{node11v}{100}
            \csdef{node12i}{\resizebox{!}{1em}{$\bullet$}}            \csdef{node12v}{60}
            \csdef{node13i}{\resizebox{!}{1em}{$\blacksquare$}}       \csdef{node13v}{80}
            \csdef{node21i}{\resizebox{!}{1em}{$\circ$}}              \csdef{node21v}{85}
            \csdef{node22i}{\resizebox{!}{1em}{$\square$}}            \csdef{node22v}{60}
            \csdef{node23i}{\resizebox{!}{1em}{$\blacktriangle$}}     \csdef{node23v}{30}
            \csdef{node31i}{}                                         \csdef{node31v}{20}
            \csdef{node32i}{\resizebox{!}{1em}{$\vartriangle$}}       \csdef{node32v}{30}
            \csdef{node33i}{\resizebox{!}{1em}{\faBuilding[regular]}} \csdef{node33v}{10}
            \newcommand{\xybaraxis}[1]{
                \draw[->] (0,0)--++(0:6) node[at end,right]{$t$}
                    coordinate[at start](Sx)
                    coordinate[midway](Mx)
                    coordinate[at end](Ex)
                    ;
                \draw[->] (0,0)--++(90:110) coordinate[midway](My) coordinate[at end](Ey);
                \node[COLpath#1] at (Mx|-Ey) {grad $\Phi$};
                \node[COLint] at (Ex|-My) {$\Phi$};
                \draw[dashed] (Sx|-max5)--(max5);
                \draw[dashed] (Sx|-max1)--(max1);
                \draw[<->] ([xshift=5pt]Sx|-max1)coordinate(P)--(P|-max5) node[sloped,midway,below,inner sep=0.5pt,font=\footnotesize]{returns};
            }
            \newlength{\barchartx}\setlength{\barchartx}{18pt}
            \newlength{\barcharty}\setlength{\barcharty}{0.9pt}
            \newlength{\barchartyos}\setlength{\barchartyos}{-43pt}
        \begin{scope}
                \matrix (m) [matrix of nodes,nodes in empty cells,row sep=40pt,column sep=40pt] {& & \\ & & \\ & & \\};
                \foreach \row in {1,2,3} {\foreach \col in {1,2,3} {\node[graphnode,fill=COLint!\csuse{node\row\col v}] (n-\row-\col) at (m-\row-\col) {\csuse{node\row\col i}};}}
                \foreach \n[remember=\n as \lastn (initially 3-3)] in {2-3,1-3,1-2,1-1} {\draw[path1] (n-\lastn)--(n-\n);}
                \foreach \n[remember=\n as \lastn (initially 3-3)] in {3-2,2-2,2-1,1-1} {\draw[path2] (n-\lastn)--(n-\n);}
        \end{scope}
        \begin{scope}[x=\barchartx,y=\barcharty,shift={(75pt,\barchartyos)}]
            \setcounter{barpos}{0}
            \foreach \n in {33,23,13,12,11} {\stepcounter{barpos}
                \node[below] at (\thebarpos,0) {\csuse{node\n i}};
                \draw[draw=COLint!\csuse{node\n v},bar] (\thebarpos,0)--++(90:\csuse{node\n v}) coordinate (max\thebarpos);
                \coordinate (pos\thebarpos) at (\thebarpos.5,0);
            }
            \setcounter{barpos}{1}
            \foreach \n in {1,2,3,4} {\stepcounter{barpos}\draw[path1] (pos\n|-max\n)--(pos\n|-max\thebarpos);}
            \xybaraxis{1}
        \end{scope}
        \begin{scope}[x=\barchartx,y=\barcharty,shift={(205pt,\barchartyos)}]
            \setcounter{barpos}{0}
            \foreach \n in {33,32,22,21,11} {\stepcounter{barpos}
                \node[below] at (\thebarpos,0) {\csuse{node\n i}};
                \draw[draw=COLint!\csuse{node\n v},bar] (\thebarpos,0)--++(90:\csuse{node\n v}) coordinate (max\thebarpos);
                \coordinate (pos\thebarpos) at (\thebarpos.5,0);
            }
            \setcounter{barpos}{1}
            \foreach \n in {1,2,3,4} {\stepcounter{barpos}\draw[path2] (pos\n|-max\n)--(pos\n|-max\thebarpos);}
            \xybaraxis{2}
        \end{scope}
    \end{tikzpicture}%

    \caption{Pure potential shaping reward functions are finitely \emph{conservative}: the returns they generate
      do not depend on the precise trajectory, only the start and end state and the length (see \cref{thm:fundamental}).
      Here, we illustrate the reason: at each step, the reward
      is a finite difference between potentials; adding up all those differences gives the overall difference
      between the first and last potential. We show the result for finite trajectories without discounting,
      but the principle remains the same with discounting ($\gamma < 1$). With discounting and a bounded potential,
      this is also true for \emph{infinite} trajectories.}\label{fig:fundamental}
\end{figure*}

An important result in vector calculus is that gradient fields are always conservative.
As \cref{fig:fundamental} illustrates, this also holds for potential shaping---this is simply
a restatement of the fact that potential shaping changes returns by a term only depending
on the first and last state. Just like in vector calculus, the converse is also true:
any conservative reward function must already be a pure potential shaping term. Formally:
\begin{restatable}{proposition}{fundamentalTheorem}\label{thm:fundamental}
    Let $\graph$ be a transition graph.
    \begin{enumerate}[(i)]
        \item The gradient of any potential $\potential$ on $\graph$ is finitely conservative.
              If $\potential$ is bounded, then its gradient is conservative.
              Specifically, for a finite trajectory $\tau = (s_0, \ldots, s_T)$,
              \begin{equation}
                \int_\tau \grad \potential = \gamma^T \potential(s_T) - \potential(s_0)\,,
              \end{equation}
              and for an infinite trajectory $\tau$ and bounded $\potential$,
              $\int_\tau \grad \potential = -\potential(s_0)$.
              \label{item:gradient_is_conservative}
        \item If $R$ is a conservative reward function on $\graph$,
              then $R$ is the gradient of some bounded potential.\label{item:conservative_is_gradient}
    \end{enumerate}  
\end{restatable}
In particular, note that a reward function is conservative if and only if it is the gradient of a bounded potential.

Part \ref{item:gradient_is_conservative} is sometimes called the \emph{fundamental theorem of calculus for line integrals}
in the context of vector calculus because it generalizes the one-dimensional fundamental theorem of calculus.
A result similar to part \ref{item:conservative_is_gradient} was also proven by \citet[Lemma~B.3]{skalse2022}.

There is a version of part \ref{item:conservative_is_gradient} for finitely conservative reward functions but
it is slightly more complicated, see \cref{thm:finitely_conservative_potential} in \cref{sec:finitely_conservative}.

\subsection{Optimality-preserving reward functions}\label{sec:optimality-preserving}

The main reason that (bounded) potential shaping and conservative reward functions are interesting
is the fact that using them as shaping terms does not affect optimal policies.

\begin{definition}
    A reward function $F$ on a transition graph
    $\graph$ is \emph{optimality-preserving} if for any MDP $M$ compatible with $\graph$ and any
    reward function $R$, $R + F$ has the same set of optimal policies in $M$ as $R$.
\end{definition}
\Cref{thm:optimality_preserving} in \cref{sec:more_results} contains a few alternative
characterizations of when a reward function is optimality-preserving.

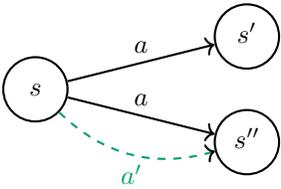
\begin{figure}
    \begin{minipage}{0.2\textwidth}
      \begin{tikzpicture}[x=40pt,y=20pt]
        \node[emptyvertex,text width=10pt,align=center] (V1) at (0,0)  {$s$};
        \node[emptyvertex,text width=10pt,align=center] (V2) at (2,1)  {$s'$};
        \node[emptyvertex,text width=10pt,align=center] (V3) at (2,-1) {$s''$};
        \draw[edge] (V1)--(V2) node[midway,above]{$a$};
        \draw[edge] (V1)--(V3) node[midway,above]{$a$};
        \draw[edgedash] (V1)to[bend right] node[midway,below,COLpath2]{$a'$} (V3);
    \end{tikzpicture}%

    \end{minipage}
    \hfill
    \begin{minipage}{0.25\textwidth}
        \caption{A transition graph has \emph{distinguishing actions} if $a'$ exists.}\label{fig:distinguishing_actions}
    \end{minipage}
\end{figure}

While conservative reward functions are always optimality-preserving,
the converse requires an additional assumption about the topology of the transition graph.
We say that a transition graph has \emph{distinguishing actions} if
for transitions $s \to s'$ and $s \to s''$ in $\graph$ with $s' \neq s''$,
there are actions $a \neq a'$ such that $(s, a, s') \in \transitions$ and
$(s, a', s'') \in \transitions$ (see \cref{fig:distinguishing_actions}).
Intuitively, this means the transition graph admits a compatible
MDP where the agent can control whether it ends up in $s'$ or $s''$.

\begin{restatable}{theorem}{conservativeOptimality}\label{thm:conservative_optimality}
    Let $F$ be a reward function on a transition graph $\graph$.
    If $F$ is conservative, then $F$ is optimality-preserving.

    If $\graph$ has distinguishing actions, the converse also holds,
    so $F$ is conservative if and only if it is optimality-preserving.
\end{restatable}

The forward direction of \cref{thm:conservative_optimality} is straightforward
and well-known (just not in the calculus language we use). The converse direction
is a significant generalization over the result from \citet{ng1999}, which requires
$\graph$ to be \emph{complete} (i.e.\ $\transitions = \states \times \actions \times \states$), rather than only having distinguishing actions.

The version by Ng et al.\ tells us that if we know \emph{absolutely nothing} about the transition dynamics
and want to ensure that a shaping term we add leaves the optimal policy set unchanged,
then we need to use potential shaping. But often, we have some partial knowledge about dynamics.
For example, we might know that we have a gridworld, and are only uncertain about
the placement of walls, or whether there is some noise. This still
massively restricts the space of possible transition dynamics, since it excludes any
\enquote{teleportation} between distant parts of state space. The results by Ng et al.\
do not tell us which shaping terms are valid given such partial knowledge. Our generalization,
\cref{thm:conservative_optimality} (together with \cref{thm:fundamental}), shows that (bounded)
potential shaping is \emph{still} the only additive shaping term
that preserves optimal policies, even given partial knowledge in the form of certain transition graphs.

\subsection{The curl}

In 3D vector calculus, a useful indicator of a vector field ${\vec{v}}$ being conservative is its \emph{curl} $\nabla \times \vec{v}$.
Conservative vector fields have zero curl, and under certain conditions the converse holds as well.
We define here a discounted analogue of curl for transitions graphs.
The curl of a reward function will not itself be a function on transitions,
but instead a function on \emph{diamonds} in the transition graph:
\begin{definition}
    Let $\graph = (\states, \actions, \transitions, \gamma)$ be a transition graph.
    A \emph{diamond} $\delta$ in $\graph$ is an ordered pair $\delta = (\delta_1, \delta_2)$ of length-two trajectories
    $\delta_1$ and $\delta_2$ in $\graph$ with the same start and end point.
    We write $\diamonds(\graph)$ for the set of all diamonds in $\graph$
    (and may drop the $\graph$ dependency if it is clear from context).
    We write $\diamondMaps$ for the space of all maps from diamonds to the reals.
\end{definition}

\begin{definition}
    Let $\graph$ be a transition graph and $\diamonds$ be the set of all diamonds in $\graph$.
    Then the \emph{curl} on $\graph$ is the map
    \begin{equation}
        \curl: \vectors \to \diamondMaps, \quad\quad (\curl R)(\delta) := \int_{\delta_1} R - \int_{\delta_2} R \,.
    \end{equation}
\end{definition}

One immediate observation is that, just like in the Euclidean setting, the curl
of (even just finitely) conservative reward functions is zero. Because of \cref{thm:fundamental},
the curl of any gradient reward functions (i.e.\ potential shaping) is zero as well.

\begin{figure}
    \begin{minipage}{0.2\textwidth}
      \begin{tikzpicture}[x=40pt,y=20pt]
        \node[emptyvertex] (V1) at (0,0)  {};
        \node[emptyvertex] (V2) at (1,1)  {};
        \node[emptyvertex] (V3) at (1,-1) {};
        \node[emptyvertex,dashed,COLpath2] (V4) at (2,0)  {};
        \draw[edge] (V1)--(V2);
        \draw[edge] (V1)--(V3);
        \draw[edgedash] (V2)--(V4);
        \draw[edgedash] (V3)--(V4);
    \end{tikzpicture}%

    \end{minipage}
    \hfill
    \begin{minipage}{0.25\textwidth}
    \caption{A transition graph is \emph{diamond-complete} if any two transitions
      with the same starting state can be completed to a diamond.}\label{fig:diamond_complete}
    \end{minipage}
\end{figure}
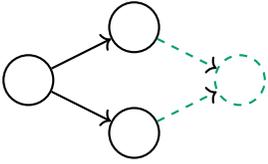
More interesting is the converse direction. Even in the Euclidean setting, not all
curl-free vector fields are conservative---this requires additional topological assumptions
on the domain (it must not have certain types of \enquote{holes}).
Something similar is true in our context, though the specific topological assumption
we need differs:
\begin{definition}
    A transition graph $\graph = (\states, \actions, \transitions, \gamma)$ is called
    \emph{diamond-complete} if for any two transitions $(s_0, a_0, s_1)$ and
    $(s_0, a_0', s_1')$ with the same starting state $s_0$,
    there are transitions $(s_1, a_1, s_2)$ and $(s_1', a_1', s_2)$.
    In other words, the two given transitions can be completed to a diamond (see \cref{fig:diamond_complete}).
\end{definition}

\begin{restatable}{theorem}{curlfree}\label{thm:curlfree}
    Let $\graph$ be a transition graph and $R$ a reward function on $\graph$. If $R$ is finitely conservative,
    then $\curl R = 0$. Conversely, if $\graph$ is diamond-complete,
    then $\curl R = 0$ implies that $R$ is conservative.
\end{restatable}

\begin{table}
    \centering
    \caption{The analogy between Lagrangian mechanics and reinforcement learning.}
    \label{tab:physics_analogy}
    \begin{tabular}{ll}
        \toprule
        Physics & Reinforcement learning \\
        \midrule
        Configuration space & State space \\
        Evolution of the physical system & Optimal trajectory \\
        Lagrangian & Reward function \\
        Action & Returns \\
        Gauge transformation & Potential shaping \\
        \bottomrule
    \end{tabular}
\end{table}

For readers interested in physics, we note that our calculus framework suggests an analogy between potential shaping and gauge transformations
in Lagrangian mechanics, shown in \cref{tab:physics_analogy}. We describe this analogy
in more detail in \cref{sec:lagrange}, where we also show how it naturally leads to rederivations of generalizations
of potential shaping to time-dependent and action-dependent potentials~\citep{devlin2012,wiewiora2003}.

\section{Divergence-free reward canonicalization}\label{sec:gauge_fixing}

In this section we describe a natural way of removing the redundancy that potential
shaping creates in reward functions. In the language of \citet{gleave2021},
we define a new \emph{canonicalization}---a map that picks one reward function
from each potential shaping equivalence class---which is useful e.g.\ when
defining distance measures between reward functions~\citep{gleave2021,wulfe2022}.

Our approach will require two new differential operators, the divergence and the Laplacian.
To define these, we will need to extend our definition of transition graphs
to include \emph{weights}, which we shall do in the following subsection.

\subsection{Divergence and Laplacian}

In vector calculus, the divergence of a vector field ${\vec{v}}$ is usually defined explicitly, by giving the formula
$\div \vec{v} := \nabla \cdot \vec{v} = \partial_1 v_1 + \ldots + \partial_n v_n$.
In our case, however, it is not immediately
obvious what the discounted version should be, e.g.\ where (if at all) discount factors should
appear. 
To find the operator most analogous to divergence in vector calculus, we can define our divergence via its relationship to the already defined gradient operator.
In vector calculus, gradient and divergence are related via integration by parts:
\begin{equation}
    \int_{\R^n} (\nabla f(x)) \cdot \vec{v}(x) \diff x = -\int_{\R^n} f(x) (\nabla\cdot\vec{v})(x) \diff x\,.
\end{equation}
Formally, we say that the divergence is the \emph{negative linear adjoint} of the gradient, meaning that $\bracket{\grad f, \vec{v}} = \bracket{f, -\div \vec{v}}$,
where $\bracket{\cdot, \cdot}$ is the $L^2$ inner product. In this abstract formulation, the same relationship holds
in (undiscounted) calculus on graphs in addition to vector calculus. So we will follow this abstract definition
using our \emph{discounted} gradient.

To define the divergence as a linear adjoint, we need an
inner product on the space of potentials and reward functions. This in turn requires
weights for states and transitions, so we first extend our definition of transition graphs:

\begin{definition}
    A \emph{finite weighted transition graph} is a tuple $(\states, \actions, \transitions, \gamma, \stateweights, \weights)$
    such that $(\states, \actions, \transitions, \gamma)$ is a transition graph,
    $\states$ and $\actions$ are finite sets,
    $\stateweights: \states \to \R_{> 0}$ is a weight function for states,
    and $\weights: \transitions \to \R_{> 0}$ is a weight function for transitions.
\end{definition}
From now on, we assume all transition graphs are weighted.
For ease of exposition, we only consider \emph{finite} weighted transition graphs in the main text.
We describe an extension to infinite state and action spaces in \cref{sec:infinite} that is more technically involved, but conceptually similar.

On the space of potentials $\scalars$, we define the inner product
\begin{equation}
    \bracket{\potential, \Psi} = \sum_{s \in \states} \stateweights(s)\potential(s) \Psi(s)\,,
\end{equation}
and on the space of rewards $\vectors$, we define
\begin{equation}\label{eq:inner_product_edges}
    \bracket{R, Q} = \sum_{t \in \transitions} \weights(t) R(t) Q(t)\,.
\end{equation}
Of course, we could have defined these inner products without weights by simply using unweighted sums,
implicitly fixing uniform weights. But while this is arguably somewhat canonical
in the discrete case, it fails entirely in the continuous setting (covered in \cref{sec:infinite}),
where the distribution defined by a uniform density would depend heavily on the parameterization
of the state and action space. We thus do not want to \enquote{hide} the fact that we need to choose
weights.

Note that we require weights $\weights(t)$ to be strictly positive.
If we want to exclude an edge, we remove it from the transition graph instead of assigning 0 weight.
This ensures $\bracket{R,Q}$ is positive-definite, a requirement to be an inner product.

We are now ready to define the divergence operator.
\begin{definition}
    The \emph{divergence} is the negative of the adjoint of the gradient, i.e.\ the
    unique linear map $\div: \vectors \to \scalars$ satisfying
    \begin{equation}
        \label{eq:divergence_adjoint}
        \bracket{R, \grad \potential} = -\bracket{\div R, \potential}
    \end{equation}
    for any reward function $R$ and potential $\potential$.
\end{definition}

While this abstract definition is often easiest to work with, we do need a concrete
formula whenever we actually want to compute the divergence:
\begin{restatable}{proposition}{explicitDivergence}
    The divergence is given explicitly by
    \begin{equation}\label{eq:divergence_explicit}
        (\div R)(s) = \frac{1}{\stateweights(s)}\left(\smashoperator[r]{\sum_{t \in \transitions_{\text{out}}(s)}} \weights(t) R(t)
        - \gamma\sum_{\mathclap{t \in \transitions_{\text{in}}(s)}} \weights(t) R(t)\right)\,,
    \end{equation}
    where $\transitions_{\text{out}}(s) \subseteq \transitions$ is the set of outgoing transitions from state $s$
    and $\transitions_{\text{in}}(s)$ the set of incoming transitions into $s$.
\end{restatable}
We can read this as the outgoing \enquote{reward flow} minus the
discounted incoming \enquote{reward flow}.

Given that we have now defined a gradient and a divergence operator, we also get
a Laplacian \enquote{for free}:
\begin{definition}
    The \emph{Laplacian} on a transition graph is the map
    $\Laplace: \scalars \to \scalars$ defined by $\Laplace := \div \circ \grad$.
\end{definition}

For $\gamma = 1$, all our definitions reduce to the usual ones
used in calculus on graphs. In particular, the Laplacian we just defined
is a generalization of the well-known graph Laplacian.
But in contrast to the usual graph Laplacian, the discounted Laplacian
is often invertible---a fact we will need later:
\begin{restatable}{lemma}{laplacianBijective}\label{thm:laplacian_bijective}
    Let $\graph$ be a finite weighted transition graph with discount factor $\gamma < 1$ and such that
    every state is part of some loop. Then the Laplace operator $\Laplace$ on $\graph$
    is bijective.
\end{restatable}

\subsection{Divergence-free reward functions}
Recall that if two reward functions $R$ and $R'$ are related by potential shaping,
i.e.\ $R' = R + \grad \potential$ for some potential $\potential$, then they are
\emph{equivalent} in a strong sense (as discussed in \cref{sec:optimality-preserving}).
It is easy to check that being related by potential shaping is indeed an
equivalence relation.

In many cases, like finding the optimal policy, we do not care what the precise reward function is.
The only thing that matters is its \emph{equivalence class}, i.e.\ the set
$\setcomp{R + \grad \potential}{\potential \in \scalars}$ of reward functions
that can be obtained by potential shaping. But working with an entire set of reward functions
is unwieldy, so what we would like is to specify a single representative from
each equivalence class. One example is measuring the distance between reward functions:
we might need a distance measure that assigns zero distance to reward functions
related by potential shaping, and an easy way to get that is to first transform
each reward function into a \enquote{canonical representative} of its equivalence
class, and then use any distance measure between these canonicalized rewards~\citep{gleave2021}.

The key fact behind the canonicalization we will introduce is that the space of reward functions
can be decomposed into divergence-free and gradient subspaces: 
\begin{restatable}{theorem}{hodge}\label{thm:hodge}
    The space of reward functions on a transition graph is the orthogonal direct sum
    of the space of divergence-free rewards and the space of potential shapings,
    \begin{equation}
        \vectors = \ker(\div) \oplus \im(\grad)\,.
    \end{equation}
\end{restatable}
Here, $\im(\grad) = \setcomp{\grad \potential}{\potential \in \scalars}$ is the image
of the $\grad$ operator and $\ker(\div) = \setcomp{R \in \vectors}{\div R = 0}$ is the kernel of the $\div$ operator.
Intuitively, \cref{thm:hodge} means that every reward function $R$ can be uniquely
written as the sum of a divergence-free reward function and the gradient of some potential.
Additionally, any divergence-free reward function is orthogonal
to any gradient, under the inner product on $\vectors$ we defined above.

We can state the decomposition more explicitly as follows, and also prove
a stronger result in cases where the Laplacian is invertible (as discussed in \cref{thm:laplacian_bijective}):
\begin{restatable}{corollary}{rewardDecomposition}\label{thm:reward_decomposition}
    For any reward function $R$,
    there is a potential $\potential$ and a unique reward function $R'$ such that
    \begin{equation}\label{eq:reward_decomposition}
        R = R' + \grad \potential
    \end{equation}
    and $\div R' = 0$.

    If furthermore $\gamma < 1$ and every state is part of a loop,
    then $\potential$ is also unique and given by $\potential = \Laplace^{-1}(\div R)$.
\end{restatable}
Note that even if $\potential$ is not unique, $\grad \potential$ certainly will be
(since $R'$ is unique and $\grad \potential = R - R'$). As a simple example
of why $\potential$ might not be unique, if $\gamma = 1$, then adding a constant
to $\potential$ doesn't affect its gradient.

As a sidenote, \cref{thm:hodge} is similar to the \emph{Helmholtz decomposition} for
vector fields, which roughly states that any vector field on $\R^3$ can be decomposed
into a divergence-free component and a gradient. A generalization of the Helmholtz
decomposition, called the \emph{Hodge decomposition} also holds in higher dimensions
and in particular on graphs~\cite{lim2020}. \Cref{thm:hodge} is the generalization
of this Hodge decomposition to discounted differential operators.

The decomposition of reward functions lets us specify representatives for each potential shaping
equivalence class:
\begin{restatable}{corollary}{divergenceFreeRepresentative}\label{thm:divergence_free_representative}
    Every potential shaping equivalence class of reward functions has exactly one
    divergence-free representative.
\end{restatable}
We write $C(R)$ for the divergence-free representative of the potential shaping
class of $R$, i.e.\ $C(R) := R'$ in \cref{eq:reward_decomposition}.
By construction, $C(R)$ is invariant under potential shaping of $R$,
i.e.\ $C(R + \grad \potential) = C(R)$.
This defines a specific representative from each
equivalence class, analogous to gauge fixing in physics.
In \cref{sec:more_results}, we show that divergence-free reward functions can also be uniquely
characterized as the reward functions with minimal $L^2$ norm within a potential shaping equivalence class.

As mentioned above, such canonicalization functions can be used to define distance measures
on reward functions that are invariant under potential shaping. \Citet{gleave2021} first introduced
this idea, but their canonicalization evaluates the reward function even on transitions known to
be impossible, which can lead to unreasonably high distances. \Citet{wulfe2022} attempt to fix
this issue, but their canonicalization still uses off-distribution transitions, and also changes
the potential shaping class of rewards. Our divergence-free canonicalization only relies on transitions that are part
of the transition graph, and never changes the potential shaping equivalence class.

\section{Conclusion}
We have introduced a framework for calculus that includes discount factors and used
this framework to reformulate results related to potential shaping in a calculus language.
This framework also allowed us to weaken the assumptions needed to show that potential shaping
is the only additive reward transformation that preserves optimal policies.
The calculus view on potential shaping furthermore suggested an analogy to physics, which we used
to naturally motivate existing generalization of potential shaping to time- and action-dependent potentials.
Finally, we defined the divergence of reward functions and showed that divergence-free reward
functions are a natural way to choose representatives from each potential shaping equivalence class.

Future work could attempt to either show that potential shaping is still the only additive
optimality-preserving shaping under even weaker assumptions than ours, or to show that further such
generalizations are impossible---this would more precisely characterize the situations under which
there are more transformations available than only potential shaping.
We also limit ourselves to applying our framework of calculus with discounting specifically
to concepts related to potential shaping. While this is the most obvious application,
given that potential shaping is the inspiration for our definition of the gradient in the
first place, it is possible that some ideas could be used in other contexts as well.

\section*{Acknowledgements}
We would like to thank Elio Farina and the Berkeley Existential Risk Initiative for helping
to create the figures in this paper. We would also like to thank Pablo Moreno, Euan McLean,
and FAR AI, as well as Daniel Filan, Scott Emmons, and Stuart Russell
for feedback on a draft of this paper and for related discussions.

\bibliography{references}

\clearpage

\appendix

\section{Finitely conservative reward functions}\label{sec:finitely_conservative}
In the main paper, we focused on the notion of conservative reward functions,
i.e.\ $\int_\sigma R = \int_\tau R$ for any \emph{infinite} trajectories $\sigma$
and $\tau$. But recall we also presented an alternative definition for \emph{finite} trajectories:
\finitelyConservativeDef*
Under the infinite-horizon assumptions we make (specifically that every state has an outgoing
transition), finite conservativeness is a weaker notion than (infinite) conservativeness:
\finitelyConservative*
\begin{proof}
    Let $\sigma$ and $\tau$ be finite trajectories with the same initial state, terminal state $s_T$,
    and length $T$. Then we can find an infinite trajectory $\zeta$ starting from $s_T$ (since every state
    has an outgoing transition). Let $\sigma \cdot \zeta$ be the concatenation of $\sigma$ with $\zeta$
    and analogously for $\tau \cdot \zeta$. We have
    \begin{equation}
        \int_\sigma R + \gamma^T \int_\zeta R = \int_{\sigma \cdot \zeta} R = \int_{\tau \cdot \zeta} R = \int_\tau R + \gamma^T \int_\zeta R\,.
    \end{equation}
    For the middle equality, we used that $R$ is conservative. Cancelling the $\gamma^T \int_\zeta R$ term on both sides,
    we then have $\int_\sigma R = \int_\tau R$, so $R$ is finitely conservative as claimed.
\end{proof}
Note that the converse implication generally doesn't hold. For example consider a transition graph with $\states = \Z$
and with a transition $s \to s'$ if and only if $\abs{s'} = \abs{s} + 1$. Starting from 0, there are two infinite trajectories,
one going in the positive direction and one going in the negative direction. For most reward functions, these will
of course have different returns, so most reward functions will not be conservative. However, given any two states
in this transition graph, there is at most a single trajectory connecting them. So any two finite trajectories
with the same initial and terminal states are in fact equal. Thus, they also have equal returns under any reward function,
so \emph{all} reward functions are finitely conservative.
This example illustrates that finite conservativeness can be an unhelpful condition from the perspective of reinforcement learning:
if the two infinite trajectories have different returns, then clearly the reward function is not optimality-preserving.

Still, under certain topological assumptions, finite conservativeness is well-behaved, and in fact equivalent to
conservativeness. First, just like conservativeness is equivalent to \emph{bounded} potential shaping,
finite conservativness is often equivalent to (possibly unbounded) potential shaping:
\begin{proposition}\label{thm:finitely_conservative_potential}
    Let $G$ be a transition graph with $0 < \gamma < 1$ and $R$ a finitely conservative reward function on $G$.
    If there is a state $s_0$ in $G$ from which any other state can be reached, and $s_0$ has a self-loop,
    then $R$ is the gradient of a potential.
\end{proposition}
\begin{proof}
    We will explicitly construct a potential $\potential$ and show that $R$ is its gradient.
    First note that $R$ is action independent, i.e.\ $R(s, a, s') = R(s, a', s')$ for all transitions $(s, a, s'), (s, a', s')$.
    To see why, interpret $(s, a, s')$ and $(s, a', s')$ as length-one trajectories with the same initial and final state;
    their returns are equal by finite conservativness, and these returns are precisely $R(s, a, s')$ and $R(s, a', s')$.
    We will therefore simply write $R(s, s')$ for the remainder of this proof.

    For any state $s$ in $G$, let $\tau_s$ be a trajectory from $s_0$ to $s$ with minimal length
    (there may be multiple minimal length trajectories, if so, pick any one of them).
    Then set
    \begin{equation}
        \potential(s) := \frac{1}{\gamma^{\abs{\tau_s}}} \left(\int_{\tau_s} R + \frac{R(s_0, s_0)}{\gamma - 1}\right)\,.
    \end{equation}
    The line integral is the same for all minimal trajectories from $s_0$ to $s$
    (since they all have the same start point, end point, and length), so this is independent of the specific choice of $\tau_s$.
    
    Note that for $s = s_0$, the shortest path is the empty one, so the integral vanishes and $\abs{\tau_s} = 0$.
    Thus, $\potential(s_0) = \frac{R(s_0, s_0)}{\gamma - 1}$, and we can also write
    \begin{equation}
        \potential(s) = \frac{1}{\gamma^{|\tau_s|}} \left(\int_{\tau_s} R + \potential(s_0)\right)\,.
    \end{equation}
    
    As an aside, let us briefly motivate this expression for the potential: the expression for $\potential(s_0)$
    can be derived immediately from the condition $(\operatorname{grad} \potential)(s_0, s_0) = \gamma \potential(s_0) - \potential(s_0) = R(s_0, s_0)$.
    Then if $R$ is indeed the gradient of a potential $\potential$, we know that
    $\int_{\tau_s} R = \gamma^{\abs{\tau_s}} \potential(s) - \potential(s_0)$,
    from which we can derive the expression above for $\potential(s)$.
    Of course this is just motivation for why this expression is the right one assuming that a potential exists at all.
    What we need to do in the remainder of this proof is show that the gradient of this potential is indeed $R$.
    
    Pick any transition $s \to s'$ in $G$. We need to show that $\gamma \potential(s') - \potential(s) = R(s, s')$.
    First, define the trajectory $\sigma$ as $\tau_s$ followed by $s \to s'$.
    This is a trajectory from $s_0$ to $s'$ of length $\abs{\tau_s} + 1$. We can decompose the line integral of $R$ along $\sigma$ as
    \begin{equation}
        \int_\sigma R = \int_{\tau_s} R + \gamma^{\abs{\tau_s}} R(s, s')\,.
    \end{equation}
    Because $R$ is conservative, we know that $\int_{\sigma} R = \int_{\sigma'} R$ for any other trajectory $\sigma'$
    from $s_0$ to $s'$ with the same length as $\sigma$. $\tau_{s'}$ is a trajectory from $s_0$ to $s'$ but might have a different length.
    Note though that $\tau_{s'}$ can never be longer than $\sigma$, since $\tau_{s'}$ was defined as a \emph{shortest} trajectory
    from $s_0$ to $s'$. Recall also that there is a self-loop $s_0 \to s_0$. Thus, we can make $\tau_{s'}$ the right length
    by first staying in $s_0$ for some number of steps. Specifically, let $\sigma'$ be the trajectory that takes a transition $s_0 \to s_0$
    (with any action) for $\abs{\sigma} - \abs{\tau_{s'}}$ steps and then follows $\tau_{s'}$.
    We then know that $\int_\sigma R = \int_{\sigma'} R$.
    
    Similar to $\int_\sigma R$ above, we can decompose $\int_{\sigma'} R$. Write $n := \abs{\sigma} - \abs{\tau_{s'}}$
    for the number of steps we stay in $s_0$, then we have
    \begin{equation}
        \int_{\sigma'} R = \sum_{t = 0}^{n - 1} \gamma^t R(s_0, s_0) + \gamma^n \int_{\tau_{s'}} R\,.
    \end{equation}
    (If $\abs{\sigma} = \abs{\tau_{s'}}$, i.e.\ $n = 0$, then the negative upper bound on the sum means that the sum is empty).
    We can write the geometric sum explicitly as $\sum_{t = 0}^{n - 1} \gamma^t = \frac{\gamma^n - 1}{\gamma - 1}$.
    This lets us rewrite the integral as
    \begin{equation}
        \begin{split}
            \int_{\sigma'} R &= \frac{\gamma^n - 1}{\gamma - 1} R(s_0, s_0) + \gamma^n \int_{\tau_{s'}} R \\
            &= (\gamma^n - 1) \potential(s_0) + \gamma^n \int_{\tau_{s'}} R\,.
        \end{split}
    \end{equation}
    Using the fact that $\int_\sigma R = \int_{\sigma'} R$, we thus get
    \begin{equation}
        \int_{\tau_s} R + \gamma^{\abs{\tau_s}} R(s, s') = (\gamma^n - 1)\potential(s_0) + \gamma^n \int_{\tau_{s'}} R\,.
    \end{equation}
    Rearranging to have $R(s, s')$ on one side, we get
    \begin{equation}
        \begin{split}
        R(s, s') &= \gamma^{n - \abs{\tau_s}} \potential(s_0) - \gamma^{-\abs{\tau_s}} \potential(s_0) \\
        &+ \gamma^{n - \abs{\tau_s}} \int_{\tau_{s'}} R - \gamma^{-\abs{\tau_s}} \int_{\tau_s} R\,.
        \end{split}
    \end{equation}
    Now recall that we defined $n := \abs{\sigma} - \abs{\tau_{s'}}$ and that furthermore $\abs{\sigma} = \abs{\tau_s} + 1$.
    Thus, $n - \abs{\tau_s} = 1 - \abs{\tau_{s'}}$. With some more rearranging on the RHS, that gives us
    \begin{equation}
        \begin{split}
            R(s, s') &= \gamma^{1 - \abs{\tau_{s'}}}\left(\int_{\tau_{s'}} R + \potential(s_0)\right) \\
            &- \gamma^{-\abs{\tau_s}} \left(\int_{\tau_s} R + \potential(s_0)\right)\,.
        \end{split}
    \end{equation}
    Note that the RHS is precisely $\gamma \potential(s') - \potential(s)$, i.e.\ the gradient of $\potential$ evaluated at $(s, s')$.
    That concludes our proof.
\end{proof}

Note that $R$ being the gradient of a potential is \emph{not} sufficient for $R$ to be optimality-preserving.
If the potential is \emph{unbounded}, then the rest term $\gamma^T \potential(s_T)$ in the returns might not
converge to zero as $T \to \infty$. The transition graph on the integers above gives an example for this: any reward function
is the gradient of a potential, but the potential might be unbounded. Just like finite conservativeness is
often too weak a condition from an RL perspective, so is potential shaping. Instead, a more well-behaved
condition is being the gradient of a \emph{bounded} potential, like we focused on in the main paper, e.g.\ \cref{fig:overview}.
Getting there from finite conservativeness requires additional assumptions:
\begin{proposition}
    If in addition to all the assumptions from \cref{thm:finitely_conservative_potential},
    there is a constant $N \in \N$ such that any state can be reached from $s_0$ in $\leq N$ steps,
    then $R$ is the gradient of a \emph{bounded} potential.
\end{proposition}
Note that we also need to assume $R$ itself is bounded, as we do throughout the paper.
\begin{proof}
    Let $C \in \R$ be a bound on $R$, i.e.\ $\abs{R(s, s')} \leq C$ for all transitions $s \to s'$ in the graph.
    Furthermore, we have $\abs{\tau_s} \leq N$ for all states $s$ (where $\tau_s$ is the shortest trajectory
    from $s_0$ to $s$ that we used to construct the potential).
    We can use these facts to bound the potential $\potential(s)$. First, note that the returns can be bounded by
    \begin{equation}
        \int_{\tau_s} R \leq \abs{\tau_s} \cdot \max_{t \in \tau_s} R(t) \leq N \cdot C\,.
    \end{equation}
    We also have $\frac{1}{\gamma^{\abs{\tau_s}}} \leq \frac{1}{\gamma^N}$. Combining these,
    \begin{equation}
        \begin{split}
            \potential(s) &\leq \frac{1}{\gamma^N}\left(N \cdot C + \frac{R(s_0, s_0)}{\gamma - 1}\right)\\
            &\leq \frac{C}{\gamma^N}\left(N + \frac{1}{\gamma - 1}\right)\,.
        \end{split}
    \end{equation}
    Thus, $\potential$ is bounded.
\end{proof}

One final observation is that for \emph{complete} transition graphs, these conditions are satisfied, so finite
conservativeness implies bounded potential shaping (and thus also conservativeness).
This is another example of a general phenomenon throughout this paper: for complete graphs,
all the properties of reward functions we discuss collapse to the same one.

\section{An analogy to Lagrangian mechanics}\label{sec:lagrange}
In this section, we will draw an analogy between potential shaping in reinforcement
learning and gauge transformations in Lagrangian mechanics---we expect this to be mainly interesting to readers familiar with physics.

The evolution of a classical physical system over time can be described by a curve $q(t)$
through the so-called \emph{configuration space} of the system. For example,
for a single point particle, the configuration space is simply the space $\R^3$
of possible positions.

Out of all the possible trajectories through configuration space, only one conforms with physical laws
and will actually be taken by the system (assuming a fixed start and end point).
For example, a particle with no forces acting on it will simply move along a straight
line, whereas a particle under the influence of gravity would move along a conic section.
The \emph{Lagrangian} of a system encodes all the necessary information to determine
which trajectories are valid. It is a function $\mc{L}(q, \dot{q}, t)$ that depends on
the configuration $q$ of the system, the time derivative $\dot{q} := \frac{dq}{dt}$, and the time $t$.

The Lagrangian determines the trajectory indirectly.
We define the \emph{action} $S[q]$ of a trajectory $q(t)$ as $S[q] := \int_0^T \mc{L}(q(t), \dot{q}(t), t)\diff t$.
Then the actual physical trajectory will have a \emph{stationary} action, meaning that infinitesimally
small changes in the trajectory with fixed start and end point do not affect the action. Typically, the action is
locally minimal with respect to the trajectory.

We can now draw the analogy to reinforcement learning: the configuration space in physics corresponds
to the state space in an MDP. In physics, we ask which trajectory through configuration space
will be physically realized, and in reinforcement learning, we instead ask which trajectories an
optimal policy takes. A physical trajectory makes the action $S$---i.e.\ the line integral over
the Lagrangian---stationary. Analogously, an optimal trajectory in RL maximizes the returns,
i.e.\ the line integral over the reward function.
\Cref{tab:physics_analogy} in the main paper summarizes this correspondence.

\emph{Gauge transformations} are transformations of the Lagrangian that only change
the action $S$ by a constant, thus leaving the time evolution of the system unchanged.
Since the action is an integral over time, such transformations can be written as
time derivatives of some arbitrary function $\potential$,
\begin{equation}
    \mc{L} \mapsto \mc{L} + \frac{d}{dt} \potential(q(t), t)\,.
\end{equation}

To transfer this idea to reinforcement learning as generally as possible,
we extend our theory to time-dependent reward functions
$R(s, a, s', t)$. Then the returns can be written as the discounted sum
\begin{equation}
    \sum_{t = 0}^{T - 1} \gamma^t R(s_t, a_t, s_{t + 1}, t)\,,
\end{equation}
which we will write as the discounted time integral
$\int_0^{T - 1} R(s_t, a_t, s_{t + 1}, t)\diff t$.
If we define discounted time derivatives as
\begin{equation}
    \frac{d}{dt}f(t) := \gamma f(t + 1) - f(t)\,,
\end{equation}
then these discounted time derivatives together with the discounted time
integral satisfy a version of the fundamental theorem of calculus,
namely $\int_0^{T - 1} \frac{df}{dt} \diff t = \gamma^T f(T) - f(0)$.

Now, just like with gauge transformations in Lagrangian mechanics,
we can add a time derivative to the reward while only changing the returns
by a constant:
\begin{equation}
    R \mapsto R + \frac{d}{dt} \potential(s_t, t)\,.
\end{equation}
Explicitly, this discounted time derivative has the form
\begin{equation}
    \frac{d}{dt} \potential(s_t, t) = \gamma \potential(s_{t + 1}, t + 1) - \potential(s_t, t)\,,
\end{equation}
which is exactly the generalization of potential shaping to time-dependent potentials
described by \citet{devlin2012}.

If we hold constant not only the initial and final \emph{state}, but the entire initial and final \emph{transition},
then we can use time derivatives of functions that depend on transitions rather than only on states:
\begin{equation}
\begin{split}
    \frac{d}{dt} \potential(\tau_t, t) &= \gamma \potential(\tau_{t + 1}, t + 1) - \potential(\tau_t, t) \\
    &= \gamma \potential(s_t, a_t, s_{t + 1}, t + 1) - \potential(s_{t - 1}, a_{t - 1}, s_t, t)\,.
\end{split}
\end{equation}
This is the potential-based advice described by \citet{wiewiora2003}.

To conclude this interlude on Lagrangian mechanics, we should note that the analogy
to reinforcement learning is certainly far from perfect. One key difference is of course
that physics has continuous time whereas time in RL is usually modelled as discrete. That is why
we need \emph{discrete} calculus to analyze potential shaping. As a consequence, some
results from physics, such as the Euler-Lagrange equations, do not carry over as easily
as the concepts discussed here. Another important note is that classical mechanics is deterministic, whereas RL often has stochastic transition dynamics and policies, which means the objective is to maximize \emph{expected returns}. Stochasticity also occurs in physics, but leads to the field of non-equilibrium statistical mechanics.

Despite these differences, we have seen that carrying over the idea of gauge transformations
to reinforcement learning naturally leads to the notion of potential shaping, as well as two
generalizations that were developed later, time-dependent and action-dependent potentials~\citep{devlin2012,wiewiora2003}.

\section{Infinite weighted transition graphs}\label{sec:infinite}
The main paper discussed gradient and line integrals on arbitrary transition graphs,
but only covered the divergence operator for transition graphs with \emph{finite} state
and transition spaces (or essentially equivalently, finite state and action spaces).
The intuitive reason for this is that the divergence operator requires summing over the
incoming and outgoing transitions of a state. While these sums could be replaced by
infinite series for \emph{countably} infinite action spaces (e.g.\ the integers),
we will in general need to replace them with integrals. This is roughly what this section
does in a formal way.

\begin{definition}
    A \emph{weighted transition graph} is a tuple $(\states, \actions, \transitions, \gamma, \stateweights, \weights)$
    such that $(\states, \actions, \transitions, \gamma)$ is a transition graph and 
    \begin{itemize}
        \item $\states$ is a measurable space of states (i.e.\ a set equipped with a $\sigma$-algebra $\stateAlgebra$),
        \item $\actions$ is a measurable space of actions equipped with a $\sigma$-algebra $\actionAlgebra$,
        \item $\transitions := \states \times \actions \times \states$ is the set of transitions,
              equipped with the product $\sigma$-algebra, which we denote as $\transitionAlgebra$.
        \item $\gamma \in [0, 1]$ is the discount factor,
        \item $\stateweights: \stateAlgebra \to [0, \infty]$ is a $\sigma$-finite measure on $\states$,
        \item $\weights: \transitionAlgebra \to [0, \infty]$ is a $\sigma$-finite measure on $\transitions$.
    \end{itemize}
\end{definition}
A potential point of confusion is that the transition space is now the \emph{entire} Cartesian
product $\states \times \actions \times \states$. Impossible transitions are modelled solely
using $\weights$ (they will not contribute to the measure of any set). This choice is purely for
convenience: in the main paper, we used a restricted transition space to avoid non-empty sets
with measure zero. But this is unavoidable now anyway, and things will become slightly easier later
if we do not have to worry about the transition space.

Note that this definition generalizes the one we gave for finite weighted transition graphs:
if $\states$ is finite, we can set $\stateAlgebra$ to the power set $2^{\states}$ of $\states$,
and similarly for $\actions$. Then the state and transition weights naturally induce measures
on $\states$ and $\transitions$.

In this section, we define $\scalars := L^2(\states, \stateweights, \R)$ and $\vectors := L^2(\transitions, \weights, \R)$,
so potentials and reward functions are not arbitrary functions anymore, but rather square-integrable
ones (under the Lebesgue integral with respect to $\stateweights$ and $\weights$). Technically, they are
\emph{equivalence classes} of square-integrable functions as is usual when defining $L^p$ spaces;
this is what gives us positive-definiteness of the inner product even though some sets of transitions
have measure zero.
Note that on finite weighted transition graphs, all potentials and reward functions are automatically
square-integrable, so this again generalizes the previous definition.

We define the usual $L^2$ inner product on $\scalars$ and $\vectors$, i.e.
\begin{equation}
    \bracket{\potential, \Psi} := \int_\states \potential \Psi d\stateweights\,,
\end{equation}
for $\potential, \Psi \in \scalars$, and analogously for $\vectors$ using $\weights$ instead of $\stateweights$.

Before introducing the divergence operator, we check that the gradient is continuous under this inner product.
This is not always the case, instead it requires an additional assumption on the measures $\stateweights$
and $\weights$:
\begin{definition}\label{def:regular}
    We call a transition graph $(\states, \actions, \transitions, \gamma, \stateweights, \weights)$
    \emph{regular} if there is a constant $C \in \R$ such that
    for all measurable sets $M \subseteq \states$, we have
    \begin{equation}
        \begin{split}
            &\weights(M \times \actions \times \states) \leq C \stateweights(M)\,,\\
            \text{and } & \weights(\states \times \actions \times M) \leq C \stateweights(M)\,.
        \end{split}
    \end{equation}
\end{definition}
Note that $M \times \actions \times \states$ and $\states \times \actions \times M$ are always measurable
since we use the product $\sigma$-algebra on the space of transitions, so this is well-defined.

\begin{lemma}\label{thm:grad_bounded}
    Let $\scalars$ and $\vectors$ be the space of potentials and reward functions
    on a regular transition graph. Then $\grad: \scalars \to \vectors$ is a bounded linear operator.
\end{lemma}
Note that the gradient operator is still defined the same way as in the main text; since it doesn't make use of the weights, infinite state spaces do not introduce any difficulties for the gradient.
\begin{proof}
    Linearity is clear. To show that the gradient is bounded, we let $p_1$ and $p_2$ be the projections
    from $\transitions$ to the current and next state respectively, i.e.\
    \begin{equation}
        p_1: \transitions \to \states, \quad p_1(s, a, s') := s\,,
    \end{equation}
    and similarly $p_2(s, a, s') := s'$. This notation allows us to write
    $\grad \potential = \gamma \potential \circ p_2 - \potential \circ p_1$. 
    Since the transition graph is regular, $\norm{\potential \circ p_1}_2 \leq \sqrt{C}\norm{\potential}_2$
    for $i \in \{1, 2\}$, where $C$ is the constant from \cref{def:regular}.
    Indeed,
    \begin{equation}
        \begin{split}
            \norm{\potential \circ p_1}^2_2 &= \int_\transitions (\potential \circ p_1)(s, a, s')^2 d\weights(s, a, s')\\
            &= \int_\transitions \potential(s)^2 d\weights(s, a, s')\\
            &= \int_\states \potential(s)^2 d\weights(s, \actions, \states)\\
            &\leq \int_\states \potential(s)^2 C d\stateweights(s) \\
            &= C\norm{\potential}^2_2\,.
        \end{split}
    \end{equation}
    
    Putting this all together, we get
    \begin{equation}
        \begin{split}
            \norm{\grad \potential}_2 &= \norm{\gamma \potential \circ p_2 - \potential \circ p_1}_2\\
            &\leq \gamma\norm{\potential \circ p_2}_2 + \norm{\potential \circ p_1}_2\\
            &\leq (1 + \gamma) \sqrt{C} \norm{\potential}_2\,.
        \end{split}
    \end{equation}
    That shows $\grad$ is bounded.
\end{proof}

That lets us again define a divergence operator as the negative adjoint of the gradient:
\begin{corollary}
    Let $\scalars$ and $\vectors$ be the spaces of potentials and reward functions
    for a regular transition graph. Then there is a unique linear operator $\div: \vectors \to \scalars$ such that
    \begin{equation}
        \bracket{\grad \potential, R} = \bracket{\potential, -\div R}
    \end{equation}
    for all potentials $\potential \in \scalars$ and rewards $R \in \rewards$.
\end{corollary}
\begin{proof}
    $\scalars$ and $\rewards$ are both Hilbert spaces (recall that we defined them as $L^2$ spaces).
    $\grad: \scalars \to \rewards$ is a bounded linear operator, so it has a (unique) linear adjoint.
\end{proof}

Similar to the finite case, we can also give a more explicit expression for the divergence:
\begin{proposition}
    Let $(\states, \actions, \transitions, \gamma, \stateweights, \weights)$ be a regular
    transition graph, and assume that $\weights$ is finite (i.e.\ $\weights(\transitions) < \infty$).
    For any measurable set $M \subseteq \states$ of states, let
    $\transitions_{\text{out}}(M) := M \times \actions \times \states$ be the set of transitions leaving
    $M$, and $\transitions_{\text{in}}(M) := \states \times \actions \times M$ the set of transitions
    into $M$. Furthermore, we define
    \begin{equation}\label{eq:lambda_definition}
        \lambda(M) := \int_{\transitions_{\text{out}}(M)} R \diff\weights - \gamma \int_{\transitions_{\text{in}}(M)} R \diff\weights\,.
    \end{equation}
    
    Then $\lambda$ is a signed measure on $\scalars$, which is absolutely continuous with respect to $\stateweights$,
    and the divergence is given by the Radon-Nikodym derivative
    \begin{equation}
        (\div R)(s) = \frac{d\lambda}{d\stateweights}(s)\,.
    \end{equation}
\end{proposition}
Before proving this claim, let us give some intuition: what we would like to do is just
replace the sums in the finite version of the divergence (\cref{eq:divergence_explicit}) with integrals.
However, we would be integrating over slices of the form $\{s\} \times \actions \times \states$, which
might very well have measure zero. We would then divide by $\stateweights(s)$, which might also very
well have measure zero, so we would not get a well-defined expression in general. Because of that,
we instead have to define the divergence as a limit (which happens implicitly via our use of the
Radon-Nikodym derivative). You should think of $M$ as a small set containing a state $s$.
Informally speaking, the divergence at $s$ is then defined by taking the limit as $M$ becomes smaller.
This is very reminiscent of one way to define the divergence in Euclidean space, namely as the net flow
through an infinitesimal sphere surrounding the point in question.

The condition that $\weights$ be finite ensures that $\lambda$ in \cref{eq:lambda_definition} is well-defined.
A weaker condition should suffice as well, but we focus on the case of finite $\weights$
to simplify matters, and since in practice $\weights$ will often be a normalized probability distribution
anyway.

One technical note: the Radon-Nikodym derivative is only unique up to a $\stateweights$-null set,
so the divergence is only really defined up to a $\stateweights$-null set as well. But that is fine
for our purposes: the divergence is supposed to be in $\scalars = L^2(\states)$ and is thus
an \emph{equivalence class} of functions that are the same up to a $\stateweights$-null set.

\begin{proof}
    $\lambda$ is indeed a well-defined signed measure, in particular $\weights$ is finite, so both integrals are finite and we never get an $\infty - \infty$ expression. That $\lambda$ is
    absolutely continuous with respect to $\stateweights$ follows immediately from the regularity
    of the transition graph: if $\stateweights(M) = 0$, then $\weights(\transitions_{\text{out}}(M)) = 0$ by regularity,
    and the same for $\transitions_{\text{in}}(M)$. So the integrals in the definition of $\lambda(M)$ are
    over $\weights$-null sets, thus $\lambda(M) = 0$.

    We now verify that the given expression is the negative adjoint to the gradient.
    Let $R \in \rewards$ be a reward function and $\potential \in \scalars$ a potential.
    We'll write $\div$ for the expression just given
    (even though so far, we haven't shown that this is in fact the divergence).
    \begin{equation}
        \begin{split}
            \bracket{\potential, \div R} &= \int_\states \potential \frac{d\lambda}{d\stateweights} \diff\stateweights\\
            &= \int_\states \potential \diff\lambda \\
            &\overset{(*)}{=} \int_\transitions R(s, a, s') \left(\potential(s) - \gamma \potential(s')\right) \diff\weights(s, a, s')\\
            &= -\bracket{\grad \potential, R}\,.
        \end{split}
    \end{equation}
    It remains to show $(*)$. We will do so using a common type of argument for showing equality of
    Lebesgue integrals: we start by considering indicator functions for $\potential$, then generalize
    to simple functions, then non-negative functions, and finally arbitrary square-integrable functions.

    First, assume that $\potential = 1_M$ for some measurable set $M \subseteq \states$. Then,
    \begin{equation}
        \begin{split}
            \int_\states \potential \diff\lambda &= \lambda(M) \\
            &= \int_{\transitions_{\text{out}}(M)} R \diff\weights - \gamma \int_{\transitions_{\text{in}}(M)} R \diff\weights \\
            &= \int_\transitions R \cdot 1_{M \times \actions \times \states} \diff\weights
            - \gamma \int_\transitions R \cdot 1_{\states \times \actions \times M} \diff\weights \\
            &= \int_\transitions R \cdot \left(1_{M \times \actions \times \states}
            - \gamma 1_{\states \times \actions \times M} \right) \diff\weights \\
            &= \int_\transitions R(s, a, s')\left(1_M(s) - \gamma 1_M(s')\right) \diff\weights(s, a, s') \\
            &= \int_\transitions R(s, a, s')\left(\potential(s) - \gamma\potential(s')\right) \diff\weights(s, a, s')\,.
        \end{split}
    \end{equation}

    The rest of the proof is a typical argument: for any simple function $\potential$ (i.e.\ a linear combination
    of indicator functions), the equality still holds because of linearity in $\potential$.
    For any non-negative measurable function $\potential$, we can then find a monotonically increasing
    sequence of simple functions $\potential_k$ that converges pointwise to $\potential$.
    By definition of the Lebesgue integral
    \begin{equation}
        \int_\states \potential \diff\lambda = \lim_{k \to \infty} \int_\states \potential_k \diff\lambda\,.
    \end{equation}
    For each $\potential_k$, we know that $(*)$ holds (since $\potential_k$ is simple), so
    \begin{equation}\label{eq:simple_potentials}
        \begin{split}
            \int_\states \potential_k \diff\lambda &= \int_\transitions R(s, a, s')
            \left(\potential_k(s) - \gamma\potential_k(s')\right) \diff\weights(s, a, s')\\
            &= \int_\transitions R(s, a, s')\potential_k(s) \diff\weights(s, a, s') \\
            &\phantom{=} - \gamma \int_\transitions R(s, a, s')\potential_k(s') \diff\weights(s, a, s')\,.
        \end{split}  
    \end{equation}
    Since $\potential_k$ is monotonically increasing, we have
    \begin{equation}\label{eq:dominated_convergence_condition}
    \begin{split}
        \abs{R(s, a, s')\potential_k(s))} &\leq \abs{R(s, a, s')}\abs{\potential_k(s)} \\
        &\leq \abs{R(s, a, s')}\abs{\potential(s)}\,.
    \end{split}
    \end{equation}
    $R$ is square-integrable by assumption. $\potential$ is square-integrable as a function on $\states$,
    and by the regularity of the transition graph, it is thus also square-integrable as a function on $\transitions$.
    So the RHS in \cref{eq:dominated_convergence_condition} is integrable. The same holds for
    $\potential(s')$ instead of $\potential(s)$. Thus, we can take the limit $k \to \infty$ in
    \cref{eq:simple_potentials} and apply dominated convergence to each of the two integrals on the RHS.
    That proves $(*)$ for non-negative $\potential$.

    Finally, we can write any square-integrable $\potential$ as a sum of its positive and negative components,
    and linearity again ensures that $(*)$ carries over. That proves $(*)$ for arbitrary
    $\potential \in \scalars$, which concludes the proof.
\end{proof}

\subsection{Divergence-free rewards}
The key result about divergence-free rewards is the decomposition as a direct sum in
\cref{thm:hodge}, $\vectors = \ker(\div) \oplus \im(\grad)$. It is still true in the infinite case that $\ker(\div)$ is the
orthogonal complement of $\im(\grad)$---that much follows immediately from the fact
that the divergence is the negative adjoint to the gradient. However, in infinite-dimensional
Hilbert spaces, a linear subspace and its orthogonal complement do not necessarily form
a direct sum for the entire space. This is only the case if the subspace in question
is closed.

We therefore need to show that $\im(\grad)$ is closed (under the $L^2$ norm on $\vectors$).
The following result gives one set of sufficient conditions for that:
\begin{lemma}
    Let $(\states, \actions, \transitions, \gamma, \stateweights, \weights)$ be a regular transition
    graph with a constant $\alpha > 0$ such that for any $M \subseteq \states$,
    $\weights(M \times \actions \times M) \geq \alpha \stateweights(M)$, and let $\gamma < 1$.
    Then $\im(\grad)$ is closed for this transition graph.
\end{lemma}
The additional condition of such an $\alpha$ existing is essentially a version of demanding the existence of self-loops,
just adapted to the infinite setting (where a single transition may very well have zero measure). Note that this is not the same as regularity---the inequality is the other way around, and we are bounding $\weights(M \times \actions \times M)$ rather than $\weights(M \times \actions \times \states)$.
\begin{proof}
    First, observe that it suffices to show there is some $C > 0$ such that
    $\norm{\grad \potential}_{2, \vectors} \geq C \norm{\potential}_{2, \scalars}$.
    The reason is as follows: let $(F_n)$ be a sequence in $\im(\grad)$ that converges
    in $\vectors$. In particular, it is a Cauchy sequence. Now pick a sequence
    $(\potential_n)$ of pre-images. Because of the existence of $C$ as above,
    $(\potential_n)$ is also Cauchy. Since $\scalars$ is complete, the sequence
    converges, $\potential_n \to \potential$. Because $\grad$ is continuous
    (see \cref{thm:grad_bounded}), we have $F_n = \grad(\potential_n) \to \grad(\potential)$,
    so the limit of $(F_n)$ lies in $\im(\grad)$.

    So now we show that $\norm{\grad \potential}_{2, \vectors} \geq C \norm{\potential}_{2, \scalars}$
    holds for some fixed $C$. To do so, let
    \begin{equation}
        M_k := \setcomp{s \in \states}{\gamma^{k/2} \leq \abs{\potential(s)} < \gamma^{(k+1)/2}}\,.
    \end{equation}
    Each $M_k$ is clearly measurable (since $\potential$ is measurable), and all $M_k$ together
    (for $k \in \Z$) form a partition of $\setcomp{s \in \states}{\potential(s) \neq 0}$ (since $\gamma < 1$).
    In particular, this means that $M_k \times \actions \times M_k$ is disjoint from $M_j \times \actions \times M_j$ for $j \neq k$,
    and that
    \begin{equation}
        \bigcup_{k \in \Z} M_k \times \actions \times M_k \subseteq \states \times \actions \times \states\,.
    \end{equation}
    Thus,
    \begin{equation}
        \begin{split}
            \norm{\grad \potential}_{2, \vectors}^2 &= \int_{\states \times \actions \times \states} (\grad \potential)^2 \diff w\\
            &\geq \sum_{k \in \Z} \int_{M_k \times \actions \times M_k} (\grad \potential)^2 \diff \weights\,.
        \end{split}
    \end{equation}
    (Note that $(\grad \potential)^2$ is of course non-negative). Intuitively, we simply throw away
    all transitions that are not between states with similar absolute potentials.

    Next, note that we can bound $(\grad \potential)^2$ below on each $M_k$. Let $s, s' \in M_k$
    and $a \in \actions$. Then
    \begin{equation}
        \begin{split}
            (\grad \potential)(s, a, s')^2 &= (\gamma \potential(s') - \potential(s))^2 \\
            &\geq (\gamma \abs{\potential(s')} - \abs{\potential(s)})^2\\
            &\geq \left(\gamma^{k/2 + 1} - \gamma^{(k+1)/2} \right)^2\\
            &= \gamma^{k + 1}(\sqrt{\gamma} - 1)^2\,.
        \end{split}
    \end{equation}
    Putting this together, we get
    \begin{equation}
        \begin{split}
            \norm{\grad \potential}_{2, \vectors}^2 &\geq \sum_{k \in \Z} \weights(M_k \times \actions \times M_k) \gamma^{k + 1}(\sqrt{\gamma} - 1)^2\\
            &\geq \sum_{k \in \Z} \alpha \stateweights(M_k) \gamma^{k + 1}(\sqrt{\gamma} - 1)^2\,,
        \end{split}
    \end{equation}
    where we used the existence of $\alpha$ from the lemma statement.
    On the other hand, we have
    \begin{equation}
        \begin{split}
            \norm{\potential}_{2, \scalars}^2 &= \int_\states \potential^2 \diff \stateweights \\
            &= \sum_{k \in \Z} \int_{M_k} \potential^2 \diff \stateweights \\
            &\leq \sum_{k \in \Z} \gamma^{k + 1} \stateweights(M_k)\,,
        \end{split}
    \end{equation}
    since $\abs{\potential(s)} < \gamma^{(k+1)/2}$ for all $s \in M_k$ by definition.

    Combining the two bounds, we can see that
    \begin{equation}
        \norm{\grad \potential}_{2, \vectors}^2 \geq \alpha (\sqrt{\gamma} - 1)^2 \norm{\potential}_{2, \scalars}^2\,.
    \end{equation}
    As argued at the beginning of this proof, this implies that $\im(\grad)$ is closed as claimed, using $C = \sqrt{\alpha} (1 - \sqrt{\gamma})$.
\end{proof}
Weaker conditions are likely also sufficient for $\im(\grad)$ to be closed.

\section{Additional results}\label{sec:more_results}
\subsection{Optimality-preserving reward functions}
\begin{proposition}\label{thm:optimality_preserving}
    Let $F$ be a reward function on a transition graph $\graph = (\states, \actions, \transitions, \gamma)$.
    Then the following are equivalent: for any MDP $M$ compatible with $\graph$ with any reward function $R$,
    \begin{enumerate}[(i)]
        \item The ordering over policies in $M$ by expected returns under
              $R$ is the same as the ordering by expected returns under $R + F$.\label{item:ordering}
        \item $R + F$ has the same set of optimal policies in $M$
              as $R$ (i.e.\ $F$ is optimality-preserving).\label{item:shaping}
        \item All policies in $M$ are optimal under $F$.\label{item:all_policies_optimal}
        \item For any reachable state $s \in \states$ and any actions $a, a' \in \actions$
              that can be taken in $s$, $Q_F^*(s, a) = Q_F^*(s, a')$.\label{item:q_star}
    \end{enumerate}
\end{proposition}
\begin{proof}
    Clearly, \ref{item:all_policies_optimal} is a special case of \ref{item:shaping}
    with $R = 0$, since all policies are optimal under the zero reward function.
    That proves $\ref{item:shaping} \implies \ref{item:all_policies_optimal}$.

    Furthermore, \ref{item:all_policies_optimal} means that any policy has the same expected
    returns under $F$, call those $C$. Then the expected returns of policy $\pi$ under
    $R + F$ are $J_R(\pi) + C$. So shaping with $F$ added only a constant offset
    and thus doesn't affect the ordering of policies, so
    $\ref{item:all_policies_optimal} \implies \ref{item:ordering}$. And clearly,
    $\ref{item:ordering} \implies \ref{item:shaping}$, so we have
    $\ref{item:shaping} \iff \ref{item:all_policies_optimal} \iff \ref{item:ordering}$.

    \ref{item:q_star} means that any action in any reachable state is optimal,
    so clearly any policy is optimal, thus $\ref{item:q_star} \implies \ref{item:all_policies_optimal}$.

    Conversely, for $\ref{item:all_policies_optimal} \implies \ref{item:q_star}$,
    let $s$ be any reachable state and $a, a'$ actions available in $s$.
    Then consider policies $\pi_1, \pi_2$ where $\pi_1(a|s) = \pi_2(a'|s) = 1$
    and $\pi_1$ and $\pi_2$ both reach $s$ with non-zero probability. If both
    $\pi_1$ and $\pi_2$ are optimal, we must have $Q^*(s, a) = Q^*(s, a')$.
\end{proof}

\subsection{Divergence-free rewards minimize $L^2$-norm within their equivalence class}
The inner product on the space $\vectors$ of reward functions induces an $L^2$ norm given by
\begin{equation}
    \norm{R}_2 := \sqrt{\bracket{R, R}} = \sqrt{\sum_{t \in \transitions} w(t) R(t)^2}\,.
\end{equation}
Divergence-free reward functions can then be characterized as minimizing the $L^2$ norm
within their potential shaping equivalence class:
\begin{restatable}{proposition}{minimalRepresentative}\label{thm:minimal_representative}
    Let $R$ be a reward function. The associated divergence-free reward function $C(R)$
    is the unique reward function such that
    \begin{equation}
        \inf_{\potential \in \scalars} \norm{R + \grad \potential}_2 = \norm{C(R)}_2\,.
    \end{equation}
\end{restatable}

This $L^2$-minimization result also implies that the function $C$ mapping to divergence-free
representatives is Lipschitz continuous with Lipschitz constant 1:
\begin{restatable}{corollary}{lipschitz}\label{thm:lipschitz}
    Let $R$ and $R'$ be reward functions, then
    \begin{equation}
        \norm{C(R) - C(R')}_2 \leq \norm{R - R'}_2\,.
    \end{equation}
\end{restatable}

\section{Proofs}\label{sec:proofs}
\subsection{Calculus on transition graphs}

\fundamentalTheorem*
\begin{proof}
    \ref{item:gradient_is_conservative} follows from a simple telescoping sum:
    \begin{equation}
        \begin{split}
            \int_\tau \grad \potential &= \sum_{t = 0}^{T - 1} \gamma^t (\gamma\potential(s_{t + 1}) - \potential(s_t)) \\
            &= \sum_{t = 0}^{T - 1} \gamma^{t + 1} \potential(s_{t + 1}) - \sum_{t = 0}^{T - 1} \gamma^t \potential(s_t) \\
            &= \sum_{t = 1}^T \gamma^t \potential(s_t) - \sum_{t = 0}^{T - 1} \gamma^t \potential(s_t) \\
            &= \gamma^T \potential(s_T) - \potential(s_0)\,.
        \end{split}
    \end{equation}
    If $\potential$ is bounded, then the first term vanishes as $T \to \infty$.

    For \ref{item:conservative_is_gradient}, we construct a potential $\potential$ as follows:
    for any state $s$, we pick an infinite trajectory $\tau^s$
    starting at $s$. Such a trajectory exists because each state in $\graph$ has at least one
    outgoing transition (as we assume throughout). Then we set
    \begin{equation}
        \potential(s) := -\int_{\tau^s} R\,.
    \end{equation}
    This is well-defined (i.e.\ independent of our choice of $\tau^s$) since $R$ is conservative.
    
    Now let $(s, a, s') \in \transitions$. Since any infinite trajectory starting at $s$ will
    lead to the same $\potential(s)$, we can assume without loss of generality that $\tau^s$
    is the concatenation of $(s, a, s')$ with $\tau^{s'}$. Then we have
    \begin{equation}
        \begin{split}
            (\grad \potential)(s, a, s') &= -\gamma \int_{\tau^{s'}} R + \int_{\tau^s} R \\
            &= -\gamma \int_{\tau^{s'}} R + \left(R(s, a, s') + \gamma \int_{\tau^{s'}} R\right) \\
            &= R(s, a, s')\,.
        \end{split}
    \end{equation}
    That shows that $R$ is indeed the gradient of $\potential$. As always, we also assume $R$ is bounded,
    i.e.\ $\abs{R(s, a, s')} \leq C$
    for some constant $C$ and all transitions $(s, a, s')$. Therefore, $\potential$ is bounded as well:
    \begin{equation}
        \begin{split}
            \abs{\potential(s)} &= \abs{\int_{\tau^s} R} \\
            &= \abs{\sum_{t = 0}^{\infty} \gamma^t R(s_t, a_t, s_{t + 1})} \\
            &\leq \sum_{t = 0}^{\infty} \gamma^t \abs{R(s_t, a_t, s_{t + 1})} \\
            &\leq \sum_{t = 0}^{\infty} \gamma^t C \\
            &= \frac{C}{1 - \gamma}\,.
        \end{split}
    \end{equation}
\end{proof}

\conservativeOptimality*
\begin{proof}
    It is clear that conservative reward functions are optimality-preserving:
    the returns depend only on the initial state by definition for conservative reward functions,
    so all policies have the same expected return, thus all policies are optimal
    under $F$. According to \cref{thm:optimality_preserving},
    that means $F$ is optimality-preserving.

    It remains to show the second claim, that if distinguishing actions exist, optimality-preserving
    (bounded) reward functions are conservative.
    Consider two infinite trajectories $\sigma$ and $\tau$ in $\graph$ with the same start point $s_0$.
    We need to show that $\int_\sigma F = \int_\tau F$.
    
    To do so, we choose transition dynamics $P$ such that for any transition $(s, a, s')$
    that is part of $\sigma$ or $\tau$, $P(s'|s, a) = 1$. For this to be possible,
    there must not be any states $s, s', s''$ and action $a$ such that $s' \neq s''$
    and $(s, a, s)$ and $(s, a, s'')$ are both part of $\sigma$ or $\tau$. The idea is
    that if two such \enquote{conflicting} transitions exist, we can set one of
    the actions to some other action $a' \neq a$ (a \enquote{distinguishing action}).
    This does not change the line integrals $\int_\sigma F$ and $\int_\tau F$ because
    $F$ is action-independent (which we prove in \cref{thm:action_independence}).
    Furthermore, we choose the initial state distribution to have support $\states$,
    to ensure that every state is reachable.

    We can now prove that under such transition dynamics, $V^*(s_0) = \int_\sigma F = \int_\tau F$.
    We write $\sigma = (s_0, a_0, s_1, \ldots)$.
    First, note that since $F$ is optimality-preserving,
    $V^*(s_0) = Q^*(s_0, a_0)$. The second step is to observe that because
    $P(s_1|s_0, a_0) = 1$, we also have $Q^*(s_0, a_0) = F(s_0, a_0, s_1) + \gamma V^*(s_1)$.
    We can continue applying these two steps to find
    \begin{equation}
        \begin{split}
        V^*(s_0) &= \sum_{t = 0}^{t - 1} \gamma^{t - 1}F(s_{t - 1}, a_{t - 1}, s_t) + \gamma^T V^*(s_t) \\
        &= \int_{\sigma_{0:t}} F + \gamma^t V^*(s_t)
        \end{split}
    \end{equation}
    for arbitrary $t$. Since the reward function $F$ is bounded by assumption and $\gamma < 1$, $V^*$ is bounded as well.
    Thus, because $\gamma < 1$, the RHS converges to $\int_\sigma F$ as $t \to \infty$. That proves
    $V^*(s_0) = \int_\sigma F$. We can show exactly analogously that $V^*(s_0) = \int_\tau F$, which concludes
    the proof.
\end{proof}

\begin{lemma}\label{thm:action_independence}
    Let $F$ be an optimality-preserving reward function on a transition graph
    $\graph = (\states, \actions, \transitions, \gamma)$.

    Then $F$ is action-independent, i.e.\ $F(s, a, s') = F(s, a', s')$ for any $s, s' \in \states$ and $a, a' \in \actions$.
\end{lemma}
\begin{proof}
    Assume that $F(s, a, s') > F(s, a', s')$ for some transitions $(s, a, s'), (s, a', s') \in \transitions$.
    Then in transition dynamics where $P(s'|s, a) = P(s'|s, a') = 1$ and $s$ is reachable,
    any policy $\pi$ that reaches $s$ with non-zero probability and has $\pi(a'|s) > 0$ cannot be optimal under $F$. But by \cref{thm:optimality_preserving}, this means that $F$ is not optimality-preserving (since then all policies would have to be optimal under $F$).
\end{proof}

\curlfree*
\begin{proof}
    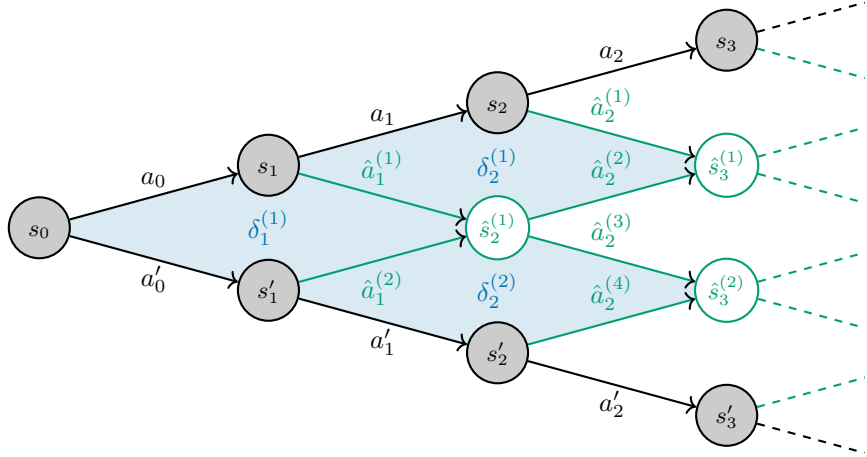
\begin{figure*}[!ht]\centering
        \begin{tikzpicture}
            \newcounter{vx}\newcounter{vy}
            \newcounter{ex}\newcounter{exn}\newcounter{ey}\newcounter{eyn}\newcounter{en}
            \def\Ssym{s}
            \matrix (v) [matrix of nodes,nodes in empty cells,row sep=17pt,column sep=80pt] {&&&&\\ &&&&\\ &&&&\\ &&&&\\ &&&&\\ &&&&\\ &&&&\\ &&&&\\ &&&&\\};
            \node[blackvertex] (V15) at (v-5-1) {\Ssym_0};
            \coordinate (B) at ($(v-9-5)!0.4!(v-8-4)$);
            \coordinate (T) at ($(v-1-5)!0.4!(v-2-4)$);
            \clip ([xshift=-2pt]V15.west|-B) rectangle (T);
            \setcounter{vx}{1}\setcounter{vy}{5}
            \foreach \vertex in {1,2,3} {\stepcounter{vx}
                \setcounter{vy}{5}\addtocounter{vy}{-\vertex}\node[blackvertex] (V\thevx\thevy) at (v-\thevy-\thevx) {\Ssym_\vertex};
                \setcounter{vy}{5}\addtocounter{vy}{\vertex}\node[blackvertex]  (V\thevx\thevy) at (v-\thevy-\thevx) {\Ssym'_\vertex};
            }
            \node[colvertex] (V35) at (v-5-3) {\hat{\Ssym}^{(1)}_2};
            \node[colvertex] (V44) at (v-4-4) {\hat{\Ssym}^{(1)}_3};
            \node[colvertex] (V46) at (v-6-4) {\hat{\Ssym}^{(2)}_3};
            \setcounter{en}{-1}
            \foreach \e in {5,4,3} {\stepcounter{en}
                \setcounter{ex}{6}\addtocounter{ex}{-\e}\setcounter{exn}{\theex}\addtocounter{exn}{1}
                \setcounter{ey}{\e}\addtocounter{ey}{-1}
                \draw[edge] (V\theex\e)--(V\theexn\theey) node[as,above]{$a_\theen$};
                \setcounter{ey}{10}\addtocounter{ey}{-\e}\setcounter{eyn}{\theey}\addtocounter{eyn}{1}
                \draw[edge] (V\theex\theey)--(V\theexn\theeyn) node[as,below]{$a'_\theen$};
            }
            \draw[edgecol] (V24)--(V35) node[ascol,above]{$\hat{a}^{(1)}_1$};
            \draw[edgecol] (V26)--(V35) node[ascol,below]{$\hat{a}^{(2)}_1$};
            \draw[edgecol] (V33)--(V44) node[ascol,above]{$\hat{a}^{(1)}_2$};
            \draw[edgecol] (V35)--(V44) node[ascol,above]{$\hat{a}^{(2)}_2$};
            \draw[edgecol] (V35)--(V46) node[ascol,above]{$\hat{a}^{(3)}_2$};
            \draw[edgecol] (V37)--(V46) node[ascol,above]{$\hat{a}^{(4)}_2$};
            \setcounter{ex}{0}
            \foreach \e in {2,4,6,8} {
                \def\edgep{COLpath2}
                \stepcounter{ex}\ifnumequal{\theex}{1}{\def\edgep{black}}{}\ifnumequal{\theex}{8}{\def\edgep{black}}{}
                \setcounter{ey}{\e}
                \addtocounter{ey}{-1}\draw[edge,-,dashed,\edgep] (V4\e)   --(v-\theey-5);
                \def\edgep{COLpath2}
                \stepcounter{ex}\ifnumequal{\theex}{1}{\def\edgep{black}}{}\ifnumequal{\theex}{8}{\def\edgep{black}}{}
                \addtocounter{ey}{2}\draw[edge,-,dashed,\edgep] (V4\e)--(v-\theey-5);
            }
                \begin{scope}[on background layer]
                    \fill[evarea] (V15.center)--(V24.center)--(V35.center)--(V26.center)--cycle;
                    \fill[evarea] (V24.center)--(V33.center)--(V44.center)--(V35.center)--cycle;
                    \fill[evarea] (V26.center)--(V35.center)--(V46.center)--(V37.center)--cycle;
                    \node[delta] at ($(V15)!0.5!(V35)$) {$\delta^{(1)}_1$};
                    \node[delta] at ($(V24)!0.5!(V44)$) {$\delta^{(1)}_2$};
                    \node[delta] at ($(V26)!0.5!(V46)$) {$\delta^{(2)}_2$};
                \end{scope}
        \end{tikzpicture}
        \caption{Proof that curl-free rewards are conservative: given two infinite
          paths, we can span the area between them with diamonds. Summing their discounted
          curls corresponds exactly to the difference in returns between the two paths,
          up to a boundary term that vanishes as we add more and more diamonds.
        }
    \end{figure*}
        
    It's clear that finitely conservative $\implies$ curl-free by the definition of the curl, what remains to show
    is the other direction. So we need to show that for two infinite paths $\sigma = (s_0, a_0, s_1, \ldots)$ and
    $\tau = (s_0, a_0', s_1', \ldots)$ with the same initial state $s_0$,
    $\int_\sigma R = \int_\tau R$.

    Using diamond-completeness, we now inductively pick states $\hat{s}_i^{(k)}$
    and actions $\hat{a}_i^{(k)}$ as follows:
    \begin{itemize}
        \item For convenience, we write $\hat{s}_i^{(0)} := s_i$, $\hat{a}_i^{(0)} := a_i$,
              $\hat{s}_i^{(i)} := s_i'$, and $\hat{a}_i^{(2i + 1)} := a_i'$.
        \item Assume we have picked states $\hat{s}_i^{(k)}$ for $k = 0, \ldots, i$ up to some $i \geq 1$,
              together with actions $\hat{a}_{i - 1}^{(2k - 1)}$ and $\hat{a}_{i - 1}^{(2k)}$ leading
              into $\hat{s}_i^{(k)}$. Then for $k = 1, \ldots i - 1$,
              we pick actions $\hat{a}_i^{(2k - 1)}$ and $\hat{a}_i^{(2k)}$
              and state $\hat{s}_{i + 1}^{(k)}$ by completing the transitions
              $\left(\hat{s}_{i - 1}^{(k - 1)}, \hat{a}_{i - 1}^{(2k - 2)}, \hat{s}_{i}^{(k - 1)}\right)$ and
              $\left(\hat{s}_{i - 1}^{(k - 1)}, \hat{a}_{i - 1}^{(2k - 1)}, \hat{s}_{i}^{(k)}\right)$ to a diamond.
              Specifically, we will have new transitions
              $\left(\hat{s}_{i}^{(k - 1)}, \hat{a}_i^{(2k - 1)}, \hat{s}_{i + 1}^{(k)}\right)$ and
              $\left(\hat{s}_{i}^{(k)}, \hat{a}_i^{(2k)}, \hat{s}_{i + 1}^{(k)}\right)$.
              We refer to this newly constructed diamond as $\delta_i^{(k)}$.
    \end{itemize}

    Our claim is now that
    \begin{equation}
        \int_\sigma R - \int_\tau R = \sum_{i = 1}^\infty \gamma^{i - 1} \sum_{k = 1}^i (\curl R)(\delta_i^{(k)})\,.
    \end{equation}
    Given that $R$ is curl-free, this would indeed conclude the proof.

    First, we write out the inner sum on the RHS more explicitly. To simplify notation, we will
    write $R_i^{(j)}$ for the reward of the transition with action $\hat{a}_i^{(j)}$. To be clear,
    many of these actions may in fact be the same, the reward does not \emph{actually} depend only
    on the action. But in our notation, the indices of the action uniquely determine the transition.
    \begin{equation}
        \begin{split}
            &\phantom{=}\;\sum_{k = 1}^i (\curl R)(\delta_i^{(k)}) \\
            &= \sum_{k = 1}^i \bigg(\phantom{\gamma}R\left(\hat{s}_{i - 1}^{(k - 1)}, \hat{a}_{i - 1}^{(2k - 2)}, \hat{s}_{i}^{(k - 1)}\right)\\[-10pt]
            &\phantom{\sum_{k = 1}^i \bigg(} + \gamma R\left(\hat{s}_{i}^{(k - 1)}, \hat{a}_i^{(2k - 1)}, \hat{s}_{i + 1}^{(k)}\right) \\[-10pt]
            &\phantom{\sum_{k = 1}^i \bigg(} - \phantom{\gamma}R\left(\hat{s}_{i - 1}^{(k - 1)}, \hat{a}_{i - 1}^{(2k - 1)}, \hat{s}_{i}^{(k)}\right) \\[-10pt]
            &\phantom{\sum_{k = 1}^i \bigg(} - \gamma R\left(\hat{s}_{i}^{(k)}, \hat{a}_i^{(2k)}, \hat{s}_{i + 1}^{(k)}\right)
            \bigg) \\
            &= \sum_{k = 1}^i \left(R_{i - 1}^{(2k - 2)} + \gamma R_i^{(2k - 1)} - R_{i - 1}^{(2k - 1)} - \gamma R_i^{(2k)}\right) \\
            &= \sum_{j = 0}^{2i - 1} (-1)^j R_{i - 1}^{(j)} + \gamma \sum_{j = 1}^{2i} (-1)^{j + 1} R_{i}^{(j)}\,.
        \end{split}
    \end{equation}
    We can now use a telescoping sum argument to get an expression for the partial sums:
    \begin{equation}
        \begin{split}
            &\phantom{=}\;\sum_{i = 1}^N \gamma^{i - 1} \sum_{k = 1}^i (\curl R)(\delta_i^{(k)}) \\
            &= \sum_{i = 1}^N \gamma^{i - 1}\left(\sum_{j = 0}^{2i - 1} (-1)^j R_{i - 1}^{(j)} + \gamma \sum_{j = 1}^{2i} (-1)^{j + 1} R_{i}^{(j)}\right) \\
            &= \sum_{i = 1}^N  \gamma^{i - 1} \sum_{j = 0}^{2i - 1} (-1)^j R_{i - 1}^{(j)} + \sum_{i = 1}^N \gamma^i \sum_{j = 1}^{2i} (-1)^{j + 1} R_{i}^{(j)} \\
            &= \sum_{i = 0}^{N - 1} \gamma^i \sum_{j = 0}^{2i + 1} (-1)^j R_i^{(j)} - \sum_{i = 1}^N \gamma^i \sum_{j = 1}^{2i} (-1)^j R_{i}^{(j)} \\
            &= \sum_{i = 0}^{N - 1} \gamma^i \left(R_i^{(0)} - R_i^{(2i + 1)}\right) - \gamma^N \sum_{j = 1}^{2N} (-1)^j R_N^{(j)}\,.
        \end{split}
    \end{equation}
    The first sum in the final line is obtained from the $j = 0$ and $j = 2i + 1$ cases of the first term
    in the penultimate line. Once these are taken care of, the two inner sums are identical, and we get a
    telescoping sum, such that only the $i = 0$ and $i = N$ terms remain. The former is an empty sum however,
    so it does not appear here.

    The first sum in this result is precisely the difference in line integrals along $\sigma$ and $\tau$
    up to time step $N$. The second term is discounted by $\gamma^N$, so as long as the reward function
    is bounded, it converges to zero (since the sum grows only linearly with $N$). Thus, the partial sums
    converge to $\int_\sigma R - \int_\tau R$, which concludes the proof.
\end{proof}

\subsection{Gauge-fixing with divergence-free rewards}
\explicitDivergence*
\begin{proof}
    Consider the indicator function
    \begin{equation}
        \delta_s(s') := \begin{cases}
            1 & \text{if } s' = s, \\
            0 & \text{otherwise}
        \end{cases}\,.
    \end{equation}
    Then we can write the divergence as
    \begin{align*}
            &\phantom{=}\;(\div R)(s) \\
            &= \frac{1}{\stateweights(s)}\bracket{\div R, \delta_s} \\
            &\overset{(\ref{eq:divergence_adjoint})}{=} -\frac{1}{\stateweights(s)}\bracket{R, \grad \delta_s} \\
            &\overset{(\ref{eq:inner_product_edges}), (\ref{eq:grad_transition})}{=} -\frac{1}{\stateweights(s)}\smashoperator[r]{\sum_{(s', a, s'') \in \transitions}} \weights(s', a, s'') R(s', a, s'')(\gamma \delta_s(s'') - \delta_s(s')) \\
            &= -\frac{1}{\stateweights(s)}\left(\gamma\smashoperator{\sum_{t \in \transitions_{\text{in}}(s)}} \weights(t) R(t) - \smashoperator{\sum_{t \in \transitions_{\text{out}}(s)}} \weights(t) R(t)\right) \\
            &= \frac{1}{\stateweights(s)}\left(\smashoperator[r]{\sum_{t \in \transitions_{\text{out}}(s)}} \weights(t) R(t)
            - \gamma\sum_{\mathclap{t \in \transitions_{\text{in}}(s)}} \weights(t) R(t)\right)
    \end{align*}
    as claimed.
\end{proof}

\laplacianBijective*
\begin{proof}
    It suffices to show that $\Laplace$ is injective, since it is a linear
    endomorphism on $L^2(\states)$, which is finite-dimensional.
    In other words, we want to show $\ker(\Laplace) = \{0\}$.
    Since $\Laplace = \div \circ \grad$, we have
    \begin{equation}
        \ker(\Laplace) = \grad^{-1}(\ker(\div)) = \grad^{-1}(\ker(\div) \cap \im(\grad))\,.
    \end{equation}
    \Cref{thm:hodge} implies that $\ker(\div) \cap \im(\grad) = \{0\}$, and thus
    $\ker(\Laplace) = \grad^{-1}(\{0\}) = \ker(\grad)$.

    Now let $\potential \in \ker(\grad)$, i.e.
    \begin{equation}
        0 = \grad(\potential)(s, a, s') = \gamma \potential(s') - \potential(s)
    \end{equation}
    for all $(s, a, s') \in \transitions$.
    Thus, if $(s, a, s') \in \transitions$, we have $\potential(s) = \gamma\potential(s')$.

    We assumed that for any $s \in \states$, there is a directed cycle
    \begin{equation}
        s \to s_1 \to \ldots \to s_{k - 1} \to s_k = s\,.
    \end{equation}
    We thus have
    \begin{equation}
        \potential(s) = \gamma\potential(s_1) = \ldots = \gamma^k\potential(s_k) = \gamma^k\potential(s)\,.
    \end{equation}
    Since $\gamma < 1$ and $k \neq 0$, this implies $\potential(s) = 0$. $s$ was arbitrary, so $\potential = 0$,
    which proves $\ker(\grad) = \{0\}$ as desired.
\end{proof}

\hodge*
\begin{proof}
    This follows immediately from the fact that the divergence is defined as the
    negative adjoint of the gradient: a reward function $R$ is orthogonal to $\im(\grad)$
    if and only if
    \begin{equation}
        0 = \bracket{R, \grad \potential} = \bracket{-\div R, \potential}
    \end{equation}
    for all $\potential \in \scalars$. This is the case iff $\div R = 0$ (by non-degeneracy
    of the inner product on $\scalars$). In summary, $R$ is orthogonal to $\im(\grad)$
    iff $\div R = 0$, i.e.\ the orthogonal complement to $\im(\grad)$ is the space
    $\ker(\div)$ of divergence-free reward functions. In finite-dimensional Hilbert spaces,
    any subspace and its orthogonal complement form an orthogonal direct sum of the entire space.
\end{proof}

\rewardDecomposition*
\begin{proof}
    According to \cref{thm:hodge}, there must be unique reward functions $R'$ and $Q$ such that
    $R = R' + Q$, $\div R' = 0$, and $Q \in \im(\grad)$. That proves the first half of the statement.

    Under the additional assumptions, $\Delta = \div \circ \grad$ is bijective according to
    \cref{thm:laplacian_bijective}, so the gradient operator must be injective.
    This means that there is a \emph{unique}
    $\potential$ such that $Q = \grad \potential$.

    For the explicit representation of $\potential$, apply the divergence operator
    to \cref{eq:reward_decomposition} to get
    \begin{equation}
        \div R = \div R' + \Laplace \potential = \Laplace \potential\,.
    \end{equation}
\end{proof}

\divergenceFreeRepresentative*
\begin{proof}
    This is simply a restatement of the first part of \cref{thm:reward_decomposition}. Existence
    follows immediately.
    For uniqueness, assume we have two reward functions
    $R$ and $R'$, both divergence-free, such that $R = R' + \grad \potential$ for some
    potential $\potential$. We can read this as a decomposition of $R$ into a divergence-free
    reward and a gradient.
    But since $R$ is itself divergence-free, another valid decomposition is $R = R + \grad 0$.
    Since the divergence-free reward in this decomposition must be unique, $R = R'$.
\end{proof}

\minimalRepresentative*
\begin{proof}
    Note that since $C(R)$ and $R$ are in the same potential shaping equivalence class,
    \begin{equation}
        \inf_{\potential \in \scalars} \norm{R + \grad \potential}_2 = \inf_{\potential \in \scalars} \norm{C(R) + \grad \potential}_2\,.
    \end{equation}
    Since $C(R) \in \ker(\div)$,
    $C(R)$ is orthogonal to $\grad \potential$ according to \cref{thm:hodge}. Thus,
    \begin{equation}
        \norm{C(R) + \grad \potential}_2^2 = \norm{C(R)}_2^2 + \norm{\grad \potential}_2^2\,.
    \end{equation}
    Clearly, this is minimized for $\grad\potential = 0$, so $C(R)$ is the representative
    of its equivalence class with minimal $L^2$ norm.
\end{proof}

\lipschitz*
\begin{proof}
    \begin{equation}
        \norm{C(R) - C(R')}_2 = \norm{C(R - R')}_2 \leq \norm{R - R'}_2\,,
    \end{equation}
    where the last step follows from \cref{thm:minimal_representative}.
\end{proof}

\section{Reward functions and vector fields}\label{sec:vectors}
We have seen that reward functions can be thought of as roughly analogous to vector fields.
For example, the domain of the divergence and curl, and the codomain of the gradient, are
all the space of reward functions. We expect that this superficial analogy is the most helpful
perspective for most readers. But for those who are interested, we now discuss some subtleties
and an arguably more accurate way to think about reward functions.

The analogy to vector fields used in the main paper is in fact a simplification:
it is better to think of reward functions as \emph{covector} fields, or even better
as \emph{1-cochains} from algebraic topology. Very roughly speaking, (simplicial)
1-chains are obtained by taking formal linear combinations of 1-simplices (i.e.\
line segments) over some ring---we will use the real numbers $\R$.
A (simplicial) 1-cochain is then a linear map from this space of 1-chains to the real numbers.
Because 1-simplices form a basis for the space of 1-chains and because 1-cochains are linear,
we can also think of a 1-cochain as an assignment of some real number to each 1-simplex.
In the context of a graph, where the 1-simplices are the edges, a 1-cochain is
thus simply a map on the edges. That is precisely how we model reward functions.

The analogy to (co)vector field comes from the fact that covector fields on manifolds,
i.e.\ differential 1-forms, are an example of 1-cochains, just not on the \emph{simplicial} chain
complex. That we can simply talk about vector fields, rather than covector fields,
is simply because the two are isomorphic in Euclidean space.

One important caveat to the entire analogy is that cochains or differential forms are
\emph{antisymmetric}, in the sense that reversing the orientation of a chain flips the
sign of its evaluation under a cochain. This is an effect of the linearity of cochains
and the fact that the orientation of a chain is represented by the signs of its coefficients.
Reward functions, on the other hand, are \emph{not} antisymmetric: in general,
$R(s, a, s') \neq R(s', a, s)$. This, along with discounting, is the reason why we needed
to define the curl on diamonds, rather than triangles (which would be the usual approach
on graphs or simplices), in order to get the same types of results as in vector calculus.

\end{document}